\documentclass[11pt,a4paper]{article}

\usepackage{graphicx, graphics, epsfig, color}
\usepackage{booktabs}
\usepackage{subcaption}
\usepackage[backend=biber,style=numeric-comp,isbn=false,doi=false,eprint=false]{biblatex}
\usepackage{amsmath,amssymb,amsthm}
\usepackage{tikz, pgfplots}
\usetikzlibrary{plotmarks,spy}
\pgfplotsset{compat=newest}
\usepackage[margin=2cm]{geometry}
\usepackage{hyperref}
\hypersetup{colorlinks=false}
\usepackage{enumitem}
\usepackage{algorithm}
\usepackage{algorithmic}
\usepackage{cleveref}
\usepackage{makecell}
\usepackage{multirow}
\bibliography{liao}

\DeclareMathOperator{\rank}{rank}
\DeclareMathOperator{\Var}{Var}

\DeclareMathOperator{\poly}{poly}

\newcommand{\T}{{\sf T}}
\DeclareMathOperator*{\argmin}{arg\,min}
\DeclareMathOperator{\spn}{span}

\newcommand{\RR}{{\mathbb{R}}}

\newcommand{\EE}{{\mathbb{E}}}

\newcommand{\A}{\mathbf{A}}
\newcommand{\B}{\mathbf{B}}
\newcommand{\C}{\mathbf{C}}
\newcommand{\X}{\mathbf{X}}
\newcommand{\U}{\mathbf{U}}
\newcommand{\V}{\mathbf{V}}
\newcommand{\Z}{\mathbf{Z}}

\newcommand{\bP}{\mathbf{P}}
\newcommand{\bS}{\mathbf{S}}
\newcommand{\M}{\mathbf{M}}
\newcommand{\bE}{\mathbf{E}}

\newcommand{\x}{\mathbf{x}}

\newcommand{\uu}{\mathbf{u}}
\newcommand{\vv}{\mathbf{v}}
\newcommand{\I}{\mathbf{I}}
\newcommand{\one}{\mathbf{1}}

\newcommand{\bTheta}{\boldsymbol{\Theta}}
\newcommand{\bSigma}{\boldsymbol{\Sigma}}
\newcommand{\bPi}{\boldsymbol{\Pi}}

\newcommand{\btheta}{\boldsymbol{\theta}}

\definecolor{RED}{rgb}{0.7,0,0}
\definecolor{BLUE}{rgb}{0,0,0.69}
\definecolor{GREEN}{rgb}{0,0.6,0}
\definecolor{PURPLE}{rgb}{0.69,0,0.8}
\definecolor{ORANGE}{rgb}{0.8,0.4,0}
\newcommand{\RED}{\color[rgb]{0.70,0,0}}
\newcommand{\BLUE}{\color[rgb]{0,0,0.69}}

\newtheorem{Definition}{Definition}
\newtheorem{Assumption}{Assumption}
\newtheorem{Theorem}{Theorem}
\newtheorem{Corollary}{Corollary}
\newtheorem{Proposition}{Proposition}
\newtheorem{Lemma}{Lemma}
\newtheorem{Remark}{Remark}

\makeatletter
\newcommand{\algorithmicoption}{\textbf{option}}
\newcommand{\OPTION}[2][default]{\ALC@it\algorithmicoption\ #2\ %
  \algorithmicdo%
  \ALC@com{#1}\begin{ALC@if}}
\newcommand{\ENDOPTION}{\end{ALC@if}}
\makeatother
\newcommand{\pms}[1]{\ensuremath{{\scriptstyle\pm #1}}}

\title{Analysis and Approximate Inference of \\ Large Random Kronecker Graphs} 

\author{%
  Zhenyu Liao\footnotemark[1]
  \and
  Yuanqian Xia\footnotemark[1]
  \and
  Chengmei Niu
  \and
  Yong Xiao\footnotemark[2]
}
\date{\today}

\begin{document}

\renewcommand{\thefootnote}{}
\footnotetext{Z.~Liao, Y.~Xia, C.~Niu, and Y.~Xiao are with the School of Electronic Information and Communications, Huazhong University of Science and Technology, Wuhan 430074, China. Y.~Xiao is also with the Peng Cheng Laboratory, Shenzhen, Guangdong 518055, China, and the Pazhou Laboratory (Huangpu), Guangzhou, Guangdong 510555, China.}
\renewcommand{\thefootnote}{\fnsymbol{footnote}}
\footnotetext[1]{Equal contribution.}
\footnotetext[2]{Author to whom any correspondence should be addressed: Yong Xiao (email: \texttt{yongxiao@hust.edu.cn}).}
\renewcommand{\thefootnote}{\arabic{footnote}}

\maketitle

\begin{abstract}
Random graph models are playing an increasingly important role in various fields ranging from social networks, telecommunication systems, to physiologic and biological networks.
Within this landscape, the random Kronecker graph model, emerges as a prominent framework for scrutinizing intricate real-world networks.
In this paper, we investigate large random Kronecker graphs, i.e., the number of graph vertices $N$ is large.
Built upon recent advances in random matrix theory (RMT) and high-dimensional statistics, we prove that the adjacency of a large random Kronecker graph can be decomposed, in a spectral norm sense, into two parts: a small-rank (of rank $O(\log N)$) signal matrix that is linear in the graph parameters and a zero-mean random noise matrix. 
Based on this result, we propose a ``denoise-and-solve'' approach to infer the key graph parameters, with significantly reduced computational complexity. 
Experiments on both graph inference and classification are presented to evaluate the our proposed method. 
In both tasks, the proposed approach yields comparable or advantageous performance, than widely-used graph inference (e.g., KronFit) and graph neural net baselines, at a time cost that scales linearly as the graph size $N$.
\end{abstract}

\section{Introduction}
\label{sec:intro}

We are living in an increasingly connected world, with a rapidly growing amount of data arising from large-scale interactive systems such as social \cite{newman2002Random,myers2014Information}, traffic, biological \cite{pavlopoulos2011Using}, and financial networks.
Graph model, in this respect, provides a natural way to describe and assess the behavior of these non-Euclidean data, e.g., how they interact with each other, in the form of pairwise relationships.

When facing large-scale networks, probabilistic graph model is a useful tool to analyze the complex behavior of entities or agents with only a small number of parameters, facilitating further analysis and graph-type data mining.
The most widely known random graph model is the Erdős–Rényi graph, for which the presence or absence of the edge between two graph vertices is modeled as an independent Bernoulli random variable with a probability parameter $p \in (0,1)$.
More advanced random graph models such as the stochastic block model~\cite{karrer2011stochastic} and Watts--Strogatz model~\cite{watts1998Collective} have then been proposed to better characterize the community structure and small-world behavior of realistic graphs.

In this paper, we investigate another popular random graph model, the random Kronecker graph model. 
It is first introduced in~\cite{leskovec2010kronecker} and applies the Kronecker product operation on a small initiator matrix $\bP_1$ (of size $m$ by $m$, say), to generate a probability matrix $\bP_K$ of much larger size (e.g., $m^K$ by $m^K$ after $K$ times of operation).
See \Cref{def:random_Kronecker_graph} below for a formal definition.
A random Kronecker graph can then be generated using $\bP_K$ with, a priori, all its structural information (such as its hierarchical structure and ``fractal'' propriety, see ~\cite{leskovec2010kronecker}) summarized in the few parameters of $\bP_1$.
Random Kronecker graph model has already shown promising potential in analyzing realistic graphs and networks in its graph statistics such as degree distributions, and/or eigenspectral behavior.

The Kronecker graph model can be used in the following two ways: 
(i) as a generative model to produce large-scale graphs that, while synthetic, closely mimic realistic graphs, and can be used for simulation and sampling purposes~\cite{leskovec2010kronecker}; and
(ii) as a prominent (low-dimensional) features extractor (e.g., the Kronecker initiator $\bP_1$) for the graph of interest, for better visualization and/or for performance in downstream tasks such as graph classification~\cite{fitzgibbon2012human,errica2019Fair}.

For the second type of application, one needs to estimate the graph parameters $\bP_1$ from a given graph.
This presents the following technical challenges:
(i) for a Kronecker graph having $N$ vertices with $N$ large, the inference of a single graph should be \emph{computationally efficient}, since downstream tasks such as graph classification may involve \emph{a non-trivial number} of such graphs; and
(ii) the Kronecker graph model naturally describes a set of isomorphic graphs (see \Cref{rem:correspondence} below), and one needs to solve the vertices matching problem, leading generally to an even heavier computational burden, if doable at all.

\subsection{Our Approach and Contribution}

The main contribution of this work is the \emph{precise} high-dimensional characterization of the Kronecker random graph model, and as a byproduct, an efficient algorithm to infer the graph parameters. 
Precisely,
\begin{enumerate}
  \item we perform, in the high-dimensional regime, a detailed analysis of the random Kronecker graph model, and show in \Cref{theo:S+N_for_A} that its adjacency follows a ``signal-plus-noise'' model with \emph{small-rank} signal matrix linear in the graph parameters of interest;
  \item we propose, in \Cref{algo:meta}, a \underline{\textbf{denoise-and-solve}} meta algorithm to approximately infer the Kronecker graph parameters from its adjacency, by first recovering the desired signal via \underline{\textbf{denoising}}, and then \underline{\textbf{solving}} the permuted linear system for the graph parameters; 
  \item we further provide, in \Cref{subsec:shrinkage}~and~\ref{subsec:permute}, use examples on the proposed meta algorithm, and then in \Cref{sec:num} numerical results to demonstrate a better performance--complexity tradeoff obtained with the proposed approach, for tasks of graph inference and graph classification~\cite{errica2019Fair}. In both tasks, the proposed approach yields comparable or advantageous performance, than widely used Kronecker graph inference and graph neural net-based methods, at a time cost that scales \emph{linearly} as the graph size $N$.
\end{enumerate}

\subsection{Related Work}

Here, we provide a brief review of related previous efforts.

\paragraph{Kronecker graph model and its applications.}
The use of Kronecker products in graphs dates back to \cite{weichsel1962kronecker}, and was largely popularized since the introduction of the Kronecker graph model~\cite{leskovec2010kronecker}.
The Kronecker graph model is a commonly adopted complex graph generator that has shown promising results in fitting many realistic networks and/or graphs, e.g., social, biological, and chemical networks, see, e.g.,~\cite{leskovec2010kronecker,reiser2022Graph}. 
It is inspired from the fact that many realistic graphs possess the self-repeating, also called \emph{multifractality}, structural property. 
Many (theoretical) properties of the random Kronecker graph model have been established in a sequence of works, e.g. the (asymptotic) size of its giant component~\cite{mahdian2010stochastic,horn2012Giant} or its degree distribution~\cite{kang2015Propertiesa,seshadhri2013indepth}.
Practical algorithms such as the KronFit approach \cite{leskovec2010kronecker} and the moment-based approach \cite{gleich2012MomentBased} are proposed to estimate the Kronecker graph parameters, with applications in, e.g., modeling human activities in videos~\cite{fitzgibbon2012human}.

\paragraph{Random graph model and random matrix theory.} 
Random graph models have attracted significant research interest in applied math, computer science, and machine learning, with applications ranging from unsupervised \cite{von2007tutorial,couillet2016kernel}, semi-supervised \cite{belkin2004regularization,zhu2005semi}, and more recently, to self-supervised learning \cite{balestriero2022Contrastive}.
Since the (eigen/singular) spectra play a crucial role in the analysis of random graph models~\cite{chung1997Spectral,chung2006Complex}, 
random matrix theory~\cite{couillet2022RMT4ML} and high-dimensional statistics~\cite{vershynin2018high}
appears as prominent tools to characterize the spectral behavior of large random graphs.
We refer the readers to \cite[Chapter~7]{couillet2022RMT4ML} for more detailed discussions on the random matrix analysis of some popular random graph models, with a focus on community detection.
Here, our result extends the line of works on the spectral analysis of dense random graphs \cite{lovasz2006Limits,abbe2018Community} to the Kronecker graph model, and propose an efficient inference method based on shrinkage estimation and permuted linear regression.


\subsection{Notations and Organization of the Paper} 

\paragraph{Notations.}
We denote scalars by lowercase letters, vectors by bold lowercase, and matrices by
bold uppercase.
We denote $\RR$ the set of real numbers and $\mathbb N$ the set of natural numbers.
For a matrix $\A$, we denote $\A^\T$ its transpose and $\A^\dagger$ its Moore–-Penrose pseudoinverse.
We use $\| \cdot \|_2$ to denote the Euclidean norm for vectors and spectral/operator norm for matrices, and denote $\| \A \|_{\max} \equiv \max_{i,j} |A_{ij}|$.
We use $\| \vv \|_0$ and $\| \vv \|_1$ to denote the $\ell_0$-~and~$\ell_1$-norm of the vector $\vv \in \RR^p$.
We denote $\one_p$ and $\I_p$ the vector of all ones of dimension $p$ and the identity matrix of dimension $p \times p$, respectively.
For two matrices $\A \in \RR^{m \times n}, \B \in \RR^{p \times q}$, we denote $\A \otimes \B \in \RR^{mp \times nq}$ their Kronecker product. 
We denote ${\rm vec}(\A) \in \RR^{mn}$ the vectorization of $\A \in \RR^{m \times n}$ by appending (in order) the columns of $\A$, and ${\rm mat}(\mathbf{a})$ the matricization of a vector $\mathbf{a}$ so that ${\rm mat}( {\rm vec}(\A) ) = \A$.
For a random variable $z$, $\EE[z]$ and $\Var[z]$ denotes the expectation and variance of $z$, respectively.
As $N \to \infty$, we use $O(\cdot)$ and $o(\cdot)$ notations as in standard asymptotic statistics \cite{van2000asymptotic}, and use $\tilde O(\cdot)$ and $\tilde o(\cdot)$ to hide terms that grows at most as ${\rm poly}(\log N)$.

\medskip 

In the remainder of the paper, we introduce, in \Cref{sec:system}, the random Kronecker graph model under study, together with our working assumption.
Our main technical results on behavior of high-dimensional random Kronecker graphs, as well as the proposed approximate inference framework, are placed in \Cref{sec:main}.
Numerical evaluations of the proposed algorithm in graph inference and classification are given in \Cref{sec:num}.
The article closes with conclusion and future perspectives in \Cref{sec:con}.

\section{System Model and Preliminaries}
\label{sec:system}

For a directed graph $G(V,E)$ having $N$ vertices, we use $\A \in \{ 0,1 \}^{N \times N}$ for its adjacency matrix, so that the edge $(i,j)$ is present in $G$ if $[\A]_{ij} = 1$, and $[\A]_{ij} = 0$ otherwise.
In this paper, we focus on the random Kronecker graph~\cite{leskovec2010kronecker} defined as follows.

\begin{Definition}[Random Kronecker graph]\label{def:random_Kronecker_graph} 
We say a graph $G$ having $N$ vertices follows a random Kronecker graph model with probability initiator $\bP_1$,
\begin{equation}
   \bP_1 = \{ P_{uv} \}_{u,v = 1}^m \in \RR^{m \times m}, \quad \quad P_{uv} \in (0,1),
\end{equation}
if the entries of its adjacency matrix $\A \in \{0,1 \}^{N \times N}$ are (up to vertex correspondence via a permutation matrix $\bPi$ of size $N$, see \Cref{rem:correspondence} below) independently drawn from a Bernoulli distribution with parameter $\bP_K \in \RR^{N \times N}$. 
This probability matrix $\bP_K$ is the $K$-th Kronecker power of $\bP_1$:
\begin{equation}\label{eq:def_bP_K}
   \bP_K = \bP_{K-1} \otimes \bP_1 = \underbrace{ \bP_1 \otimes \ldots \otimes \bP_1 }_{ K~\text{times} } = \bP_1^{\otimes K } , 
\end{equation}  
with $N = m^K$ and $K \in \mathbb N$. 
That is, for $i,j \in \{1,\ldots,N\}$,
\begin{equation}\label{eq:def_Z}
    [\A]_{ij} \sim {\rm Bern}( [\bP_K]_{ij}),~\text{and}~[\A]_{ij} = [\bP_K]_{ij} + [\Z]_{ij},
\end{equation}
for $\Z \in \RR^{N \times N}$ having independent entries of zero mean and variance $[\bP_K]_{ij}(1- [\bP_K]_{ij})$.
\end{Definition}

The Kronecker graph model in \Cref{def:random_Kronecker_graph} is parameterized by the probability matrix $\bP_K \in \RR^{N \times N}$.
Note that ({\romannumeral 1}) we have $N = m^K$ so that $K = \log_m N$; and ({\romannumeral 2}) the entries of $\bP_K$ are polynomial in the entries of $\bP_1$, i.e., $[\bP_K]_{ij} = \poly(P_{uv}) \in (0,1)$, and \emph{only} depends on $K$ and the probability initiator $\bP_1 \in \RR^{m \times m}$ per \eqref{eq:def_bP_K}.

The Kronecker graph model in \Cref{def:random_Kronecker_graph} naturally defines a set of isomorphic graphs, by exchanging the indices of its vertices.
This is discussed in the following remark.
\begin{Remark}[Vertices matching]\normalfont
\label{rem:correspondence}
For a Kronecker graph $G$ having $N$ vertices in \Cref{def:random_Kronecker_graph}, with each vertex having a unique label\footnote{This can be considered as the vertex index, which does not carry any particular information about the vertex, but just uniquely identifies the vertex.}, one has
\begin{equation}
  \A = \bPi (\bP_K + \Z) \bPi^{-1},
\end{equation}
for some permutation matrix $\bPi \in \mathcal P_N$, with $\mathcal P_N$ the set of all permutation matrices of dimension $N$ by $N$.
As such, the inference of $G$ from its random adjacency $\A$ comprises both the inference of
({\romannumeral 1}) the probability matrix $\bP_K$ (or equivalent $\bP_1$ with known $K$); \emph{and} 
({\romannumeral 2}) the vertex correspondence uniquely determined by $\bPi$.
\end{Remark}

To infer the Kronecker graph parameters from its random adjacency $\A$, we position ourselves under the following re-parameterization on the Kronecker initiator $\bP_1$.

\begin{Assumption}[Re-parameterization of graph initiator]\label{ass:growth-rate}
We have, for fixed $m$ and as $N \to \infty$ that, the entry $P_{uv}$ of the initiator $\bP_1$ can be re-parameterized as
\begin{equation}
    P_{uv} = p + X_{uv}/\sqrt N,
\end{equation}
for $p \in (0,1)$ and $\X \equiv \{ X_{uv} \}_{u,v=1}^m$ with $ \| \X \|_{\max} = O(1)$.
\end{Assumption}

\begin{Remark}[On \Cref{ass:growth-rate}]\normalfont
Taking $X_{uv} = 0$ in \Cref{ass:growth-rate} one gets the popular Erd\H{o}s--Rényi graph with probability parameter $\bar p = p^K$. 
In this respect, \Cref{ass:growth-rate} says that the Kronecker graph under study is an extension to the Erd\H{o}s--Rényi graph model, with $\X \in \RR^{m \times m}$ that characterize its deviation from the Erd\H{o}s--Rényi model. 
While large-and-sparse and dense-and-small are two commonly adopted approaches for simplifying realistic graphs, it has already been reported that many extremely large realistic networks and/or graphs do contain important large-and-dense subgraphs, see \cite{gibson2005Discovering, danisch2017large, chekuri2022densest} for the example of connection graphs between hosts on the World Wide Web, in which there exist several hundred giant dense subgraphs of at least ten thousand hosts. 
Approximating these graphs/networks based on large-and-sparse or dense-and-small assumptions may ignore some specific structural attributes that can be critical for graph/network analysis.
\end{Remark}

\section{Main Results}
\label{sec:main}

Having introduced the Kronecker graph in \Cref{def:random_Kronecker_graph} and our working \Cref{ass:growth-rate}, we present now our main results.
We perform, in \Cref{subsec:analysis}, a detailed analysis of the random Kronecker graph, and show in \Cref{theo:S+N_for_A} that its adjacency $\A$ is, up to permutation, the sum of a small-rank ``signal'' matrix (that is linear in the graph parameters $\X$ of interest) and a random ``noise'' matrix; 
this allows us to propose, in \Cref{algo:meta}, a \underline{\textbf{denoise-and-solve}} meta algorithm to approximately infer the Kronecker graph parameters.
We then provide, in \Cref{alg:shrinkage_estim}~and~\Cref{alg:permuted_LR} of \Cref{subsec:shrinkage}~and~\Cref{subsec:permute}, respectively, concrete use examples for the proposed inference framework.

\subsection{Analysis of Kronecker graphs and a meta algorithm}
\label{subsec:analysis}

Our objective is to estimate the Kronecker graph initiator $\bP_1$ (so both $p$ and $\X \in \RR^{m \times m}$ under the re-parameterization in \Cref{ass:growth-rate}) from a random realization of the graph adjacency $\A$.
To this end, we define, for $\bS_1 = \X/N$ and $k \in \{ 2, \ldots, K\}$, the following sequence of matrices $\bS_1, \ldots, \bS_K$ as
\begin{equation}\label{eq:def_S}
     \bS_k = \frac{p^{k-1}}N ( \one_{m^{k-1}}\one_{m^{k-1}}^\T) \otimes \X + p \bS_{{k-1}} \otimes (\one_m\one_m^\T).
\end{equation}
We then show that the $K$-th Kronecker power $\bP_K = \EE[\A]$ of $\bP_1$ is closely connected to $\bS_K$ defined above (which is then closely related to the graph parameters $\X$) and is of small rank with respect to its dimension $N$. 
This is described in the following result, and proven in \Cref{sec:appendix_proof_of_linearized_rank}.

\begin{Proposition}[Approximate small-rankness of $\bP_K$]\label{prop:linearized}
Under \Cref{ass:growth-rate} and for $N $ large, we have, for $\bP_K \in \RR^{N \times N}$ the $K$-th Kronecker power of $\bP_1$ as in \eqref{eq:def_bP_K} that:
\begin{itemize}
  \item[({\romannumeral 1})] $ \| \bP_K - \bP_K^{\rm lin} \|_{\max} = \tilde O(N^{-1})$ and $ \| \bP_K - \bP_K^{\rm lin} \|_2 = \tilde O(1)$ for a \emph{linearized} $\bP^{\rm lin}_{K}$ defined as
  \begin{equation}\label{eq:def_P_lin}
    \bP^{\rm lin}_{K} \equiv p^K\one_{N}\one_{N}^\T + \sqrt{N} \bS_K,
  \end{equation}
  with $\bS_K$ in \eqref{eq:def_S} for $k =K$ so that $\| \bS_K \|_2 = \tilde O(1)$; and
  \item[({\romannumeral 2})] $\bS_K$ is \emph{linear} in (the entries of) $\X$, in the sense that
  \begin{equation}\label{eq:S_linear_propo}
      \bS_K = {\rm mat}(\bTheta{\rm vec}(\X)) \in \RR^{N \times N}, 
  \end{equation}
  for known coefficients $\bTheta \in \RR^{N^2 \times m^2}$ (from binomial expansion) such that $\| \bTheta \|_{\max} = \tilde O(N^{-1})$ with $\bTheta \one_{m^2} = \frac{p^{K-1} K}N \one_{N^2}, \bTheta^\T \one_{N^2} = \frac{p^{K-1}KN}{m^2} \one_{m^2}$; and
  \item[({\romannumeral 3})] $\max( \rank(\bS_K),\rank(\bP_K^{\rm lin}) )  \leq (m-1)K+1$.
\end{itemize}
\end{Proposition}

As a consequence of \Cref{prop:linearized}, we have the following signal-plus-noise decomposition on the random Kronecker adjacency $\A$, the proof of which is given in \Cref{subsec:proof_info_plus_noise_A}.

\begin{Theorem}[Signal-plus-noise decomposition for $\A$]\label{theo:S+N_for_A}
Under \Cref{ass:growth-rate} and let $p^K \equiv \bar p \in (0,1)$, the adjacency $\A$ of a Kronecker graph in \Cref{def:random_Kronecker_graph} satisfies, for $N$ large, $\| \A \|_2 = \tilde O(\sqrt N)$ and
\begin{equation}
  \| \A - ( \bPi \bP_K^{\rm lin} \bPi^{-1} + \Z ) \|_2 = \tilde O(1),
\end{equation}
with
\begin{equation}\label{eq:def_S_K^Pi}
    \bPi \bP_K^{\rm lin} \bPi^{-1} = p^K\one_{N}\one_{N}^\T + \sqrt{N} \underbrace{\bPi \bS_K \bPi^{-1}}_{ \equiv \bS_K^{\bPi} },
\end{equation}
for some permutation matrix $\bPi$ (that corresponds to the vertex matching, see \Cref{rem:correspondence}), random matrix $\Z \in \RR^{N \times N}$ having independent entries of zero mean and variance $\bar p (1- \bar p)$, 
and linearized probability matrix $\bP_K^{\rm lin}$ in \eqref{eq:def_P_lin}.
\end{Theorem}

A direct consequence of \Cref{prop:linearized}~and~\Cref{theo:S+N_for_A} is that the key probability parameter $p$ in \Cref{ass:growth-rate} can be consistently estimated from the adjacency $\A$ as follows, proven in \Cref{subsec:proof_lem_estimate_p}.
\begin{Lemma}[Consistent estimation of $p$]\label{lem:estimate_p}
Under the notations and settings of \Cref{theo:S+N_for_A}, we have $ \one_N^\T \A \one_N/N^2  - p^K \to 0$ almost surely as $N \to \infty$.
\end{Lemma}

In plain words, \Cref{prop:linearized}~and~\Cref{theo:S+N_for_A} tells us that, the adjacency matrix $\A$ of a large Kronecker random graph can be decomposed, in a spectral norm sense, as the sum of some zero mean random matrix $\Z$ and (up to permutation by $\bPi$ and the constant matrix $p^K \one_{N}\one_{N}^\T$ that can consistently estimated per \Cref{lem:estimate_p}) some deterministic ``signal'' matrix $\bS_K$ defined in \eqref{eq:def_S}. 
In particular, this signal matrix $\bS_K$:
\begin{itemize}
  \item[({\romannumeral 1})] enjoys the property of having small rank (as a consequence of Item~({\romannumeral 3}) of \Cref{prop:linearized}, compared to the random $\Z$), and can be ``extracted'' from the noisy observation $\A$ via some \underline{\textbf{denoising}} procedure; and 
  \item[({\romannumeral 2})] is \emph{linear} in the entries of $\X$ with \emph{known} coefficients $\bTheta$, so that a perturbed linear regression allows to \underline{\textbf{solve}} the desired $\X$ from $\bS_K$ (or from its estimate).
\end{itemize}
This leads to the two-step ``\underline{\textbf{denoise-and-solve}}'' meta algorithm in \Cref{algo:meta} for random Kronecker graph inference.

\begin{algorithm}[btb]
   \caption{Meta-algorithm: approximate inference of random Kronecker graph parameters}
   \label{algo:meta}
\begin{algorithmic}[1]
   \STATE {\bfseries Input:} Adjacency matrix $\A$ of a random Kronecker graph of size $N$ as in \Cref{def:random_Kronecker_graph}.
   \STATE {\bfseries Output:} Estimates $\hat p$ and $\hat \X$ of the graph parameters $p \in \RR$ and $\X \in \RR^{m \times m}$ in \Cref{ass:growth-rate}.
   \STATE Estimate $p$ as $\hat p = \sqrt[K]{ \one_N^\T \A \one_N/N^2 }$ from \Cref{lem:estimate_p}.
   \STATE \underline{\textbf{Denoise}} the adjacency $\A$ to get an estimate $\hat \bS_K$ of $\bS_K^{\bPi}$ defined in \eqref{eq:def_S_K^Pi} with, e.g., the shrinkage estimator in \Cref{alg:shrinkage_estim}.
   \STATE \underline{\textbf{Solve}} a permuted linear regression problem (see \eqref{eq:permute_LS} below for detailed expression) to obtain $(\hat \bPi, \hat \x)$ from $\hat \bS_K$ via, e.g., the convex relaxation or the iterative hard thresholding approach in \Cref{alg:permuted_LR}.
   \STATE \textbf{return} $\hat p$ and $\hat \X = {\rm mat}(\hat \x)$.
\end{algorithmic}
\end{algorithm}

\subsection{Kronecker denoising with shrinkage estimator}
\label{subsec:shrinkage}

Here, we provide an example to algorithmically implement the \underline{\textbf{denoising}} step in \Cref{algo:meta}.
To \underline{\textbf{denoise}} the random adjacency $\A$ and recover the informative small-rank matrix $\bS_K$ (and eventually the graph parameters $\X$), we introduce the ``centered'' adjacency matrix $\bar \A$ as\footnote{This is to be distinguished from the normalized adjacency for undirected graphs, see, e.g., \cite{coja2010finding}.}
\begin{equation}\label{eq:def_centered_adj}
   \bar \A \equiv \frac1{\sqrt N} \left(  \A - \frac{ \one_N^\T \A \one_N }{N^2} \one_N \one_N^\T \right),
\end{equation}
and show, in the following result, that the centered adjacency $\bar \A$ also follows a signal-plus-noise model by removing the undesired and non-informative constant matrix of $p^K \one_N \one_N^\T$ from $\A$, the proof of which is given in \Cref{subsec:proof_prop_approx_centered_adj}.

\begin{Proposition}[Signal-plus-noise decomposition for $\bar \A$]\label{prop:approx_centered_adj}
Under \Cref{ass:growth-rate}, assume $\X \equiv \{ X_{uv} \}_{u,v=1}^m$ is ``centered'' so that $\sum_{u,v=1}^m X_{uv} = \tilde O(N^{-1/2})$. 
Then, the centered adjacency matrix $\bar \A$ defined in \eqref{eq:def_centered_adj} satisfies
\begin{equation}
    \| \bar \A - ( \bS_K^{\bPi} + \Z /\sqrt N  ) \|_2 = \tilde O(N^{-1/2}),
\end{equation}
with small-rank $\bS_K^{\bPi}$ defined in \eqref{eq:def_S_K^Pi} and random matrix $\Z$.
\end{Proposition}
As a consequence of \Cref{prop:approx_centered_adj}, to recover the desired signal matrix $\bS_K^{\bPi}$ from the noisy $\bar \A$, we resort to the following optimization problem 
\begin{equation}\label{eq:shrinkage_estimate}
\begin{aligned}
    \min_{\bS_K^{\bPi} \in \RR^{N \times N} }& \quad \left\| \bar \A - \bS_K^{\bPi} \right\|, \\ 
    \text{s.t.}& \quad \rank(\bS_K^{\bPi}) \leq (m-1) K + 1,
\end{aligned}
\end{equation}
for some matrix norm $\| \cdot \|$ that can be the Frobenius $\| \cdot \|_F$, spectral $\| \cdot \|_2$, or nuclear norm $\| \cdot \|_*$, where the rank constraint is due to Item~({\romannumeral 3}) of \Cref{prop:linearized}.

When the rank of $\bS_K^{\bPi}$ is known, the default technique to solve \eqref{eq:shrinkage_estimate} is the hard thresholding singular value decomposition (SVD) estimator given by
\begin{equation}\label{eq:def_hard_thresholding}
    \hat{\bS}_K^{\bPi} = \textstyle \sum_{i=1}^{\rank(\bS_K^{\bPi})} \hat\sigma_i \hat \uu_i \hat \vv_i^\T,
\end{equation}
with $(\hat\sigma_i, \hat \uu_i, \hat \vv_i)$ the triple of singular values (listed in a decreasing order) and left and right singular vectors of $\bar \A$.

More generally, we define the \emph{shrinkage estimator} that extends the hard thresholding SVD in \eqref{eq:def_hard_thresholding} as follows,
\begin{equation}
    \hat \bS_K^{\bPi} = \textstyle \sum_{i=1}^N f(\hat \sigma_i) \hat \uu_i \hat \vv_i^\T,~\text{with}~f\colon \RR_{\geq 0} \to \RR_{\geq 0},
\end{equation}
for some nonlinear function $f$.
The hard thresholding SVD in \eqref{eq:def_hard_thresholding} is a special case of shrinkage estimator with $f(t) = t $ for the largest $\rank(\bS_K^{\bPi})$ singular values of $\bar \A$ and zero otherwise.
This specific choice of hard thresholding function is, however, of limited interest since it requires additional efforts to be practically implemented when $\rank(\bS_K^{\bPi})$ is unknown and needs to be determined. 

While the signal-plus-noise model of the type $\bS_K^{\bPi} + \Z /\sqrt N$ for small rank $\bS_K^{\bPi}$ and random matrix $\Z$ having i.i.d.\@ zero-mean entries has been widely studied in the literature of RMT and high-dimensional statistics, previous efforts only focus on the case of fixed rank for $\bS_K^{\bPi}$ as $N \to \infty$ and does not apply to Kronecker graph inference for which the rank of $\bS_K^{\bPi}$ \emph{may grow with $N$ in a logarithmic fashion}, see again Item~({\romannumeral 3}) in \Cref{prop:linearized}.

We further provide, in \Cref{sec:spectral_analysis_centered_adjacency}, a detailed analysis of the singular spectrum of $\bar \A$ for small but growing rank of $\bS_K^{\bPi}$. 
This further leads to, by carefully adapting the proof of \cite[Theorem~1]{gavish2017optimal} to the Kronecker graph in \Cref{def:random_Kronecker_graph}, the \underline{\textbf{denoising}} shrinkage estimator $\hat \bS_K$ of $\bS_K^{\bPi}$.
This estimator $\hat \bS_K$ will then be used in \Cref{subsec:permute} below to infer the Kronecker graph parameter $\bP_1$.

The shrinkage estimator in \Cref{alg:shrinkage_estim} is more interesting than, e.g., the naive hard thresholding SVD approach in that: ({\romannumeral 1}) it uses a truncation threshold $2\sqrt{ \hat{\bar p} (1- \hat{\bar p}) }$ that can be predetermined from the graph adjacency and ({\romannumeral 2}) it yields the \emph{minimum} (asymptotic) Frobenius norm error among all shrinkage estimators of the form \eqref{eq:shrinkage_estimate}, see again \Cref{sec:spectral_analysis_centered_adjacency} for a detailed discussion on this.

\begin{algorithm}[tb]
   \caption{Shrinkage estimator of $\bS_K$ to \underline{\textbf{denoise}} $\A$}
   \label{alg:shrinkage_estim}
\begin{algorithmic}[1]
   \STATE {\bfseries Input:} Adjacency $\A $ of a random Kronecker graph having $N$ vertices as in \Cref{def:random_Kronecker_graph}.
   \STATE {\bfseries Output:} Shrinkage estimator $\hat \bS_K$ of the (permuted) signal matrix $\bS_K^{\bPi}$ defined in \eqref{eq:def_S_K^Pi}.  
   \STATE Compute the ``centered'' adjacency $\bar \A$ as in \eqref{eq:def_centered_adj}.
   \STATE Estimate $\bar p $ with $\hat{\bar p} = \one_N^\T \A \one_N/ N^2$ as in \Cref{lem:estimate_p}.
   \STATE \textbf{return} $\hat \bS_K = \sum_{i=1}^{ (m-1)\log_m(N) + 1} f(\hat \sigma_i) \hat \uu_i \hat \vv_i^\T$, with $(\hat\sigma_i, \hat \uu_i, \hat \vv_i)$ the triple of singular values (in decreasing order) and singular vectors of $\bar \A$, for $f(t) = \sqrt{t^2 - 4 \hat{\bar p} (1- \hat{\bar p}) } \cdot 1_{t > 2 \sqrt{ \hat{\bar p} (1- \hat{\bar p}) }  }$.
\end{algorithmic}
\end{algorithm}

\subsection{Kronecker solving via permuted linear regression}
\label{subsec:permute}

Having obtained the estimate $\hat \bS_K$ of the permuted signal matrix $\bS_K^{\bPi}$ using, say the shrinkage estimation in \Cref{alg:shrinkage_estim} of \Cref{subsec:shrinkage}, we now discuss how to \underline{\textbf{solve}} for the Kronecker graph parameter $\X$ from this estimate $\hat \bS_K$.

Note that we have, up to permutation by $\bPi$, that $\hat \bS_K \simeq \bPi \bS_K \bPi^{-1}$ for $N$ large, with
\begin{equation}
    \hat \bS_K \simeq \bPi \bS_K \bPi^{-1} = {\rm mat}( (\bPi \otimes \bPi) \bTheta{\rm vec}(\X)),
\end{equation}
by \Cref{lem:vec_Kron} in \Cref{sec:SM-lemmas}, with \emph{known} coefficients $\bTheta$ (from binomial expansion, see again \Cref{prop:linearized}) and \emph{unknown} permutation $\bPi$.
To recover both ${\rm vec (\X)}$ and $\bPi$, we sort to the following optimization problem:
\begin{equation}\label{eq:permute_LS}
    (\hat \bPi, \hat \x) = \argmin_{\bPi \in \mathcal P_N, \x \in \RR^{m^2}} \| (\bPi \otimes \bPi) \bTheta \x - {\rm vec}(\hat \bS_K) \|_2^2,
\end{equation}
for $\mathcal P_N$ the set of permutation matrices of size $N$.

The optimization problem of the type \eqref{eq:permute_LS} is known in the literature as linear regression with ``broken samples,'' or \emph{permuted linear regression}.
This problem is known to be extremely challenging (in fact proven to be NP-hard unless in some trivial cases) and has attracted significant research interest by, e.g., considering different simplifying statistical assumptions on the noise (the entries of $\hat \bS_K - \bS_K^{\bPi}$ in the context of this paper) and/or the coefficient matrix $\bTheta$, see for example \cite{nasrabadi2011Robust,she2011Outlier,hsu2017linear,pananjady2018linear,peng2020linear,slawski2019linear}.

Denote $\bPi_* \in \mathcal P_N$ the permutation matrix that corresponds to the \emph{true} matching of the $N$ vertices, and $d_H(\bPi_*, \I_N) \equiv |\{ i : [\bPi_*]_{ii} = 0\} |$ the Hamming distance between $\bPi_*$ and the identity matrix (which characterizes the number of mismatched vertices), the permuted linear regression problem in \eqref{eq:permute_LS} then writes
\begin{equation}\label{eq:permute_LS_1}
\begin{aligned}
    \min_{\bPi \in \mathcal P_N, \x \in \RR^{m^2}}& \quad \|{\rm vec}(\hat \bS_K) -  (\bPi \otimes \bPi) \bTheta \x  \|_2^2, \\ 
    \text{s.t.}& \quad d_H(\bPi,\I_N) \leq s,
\end{aligned}
\end{equation}
for $\bPi \in \mathcal P_N$, $\x \in \RR^{m^2}$, and some auxiliary (sparsity) variable $s \leq N$.
The optimization problem in \eqref{eq:permute_LS_1} is known to be NP-hard as long as $s = O(N)$, unless in the trivial case of $m = 1$, see~\cite{pananjady2018linear}.

To solve efficiently the Kronecker inference problem, here we consider the setting where the permutation is \emph{sparse} (so that $s \ll N$ in \eqref{eq:permute_LS_1}), and first relax the constraint in \eqref{eq:permute_LS_1} as $d_H( \bPi \otimes \bPi, \I_{N^2} ) \leq 2sN - s^2 \leq 2sN$.
Introducing $\mathbf{d} = (\bPi \otimes \bPi - \I_{N^2})\bTheta \x \in \RR^{N^2}$, the problem in \eqref{eq:permute_LS_1} can be relaxed as
\begin{equation}\label{eq:permute_LS_2}
\begin{aligned}
    \min_{\x \in \RR^{m^2},~\mathbf{d} \in \RR^{N^2} }& \quad  \|{\rm vec}(\hat \bS_K) - \bTheta \x - \mathbf{d}  \|_2^2, \\ 
    \text{s.t.}&  \quad \| \mathbf{d} \|_0 \leq 2 s N,
\end{aligned}
\end{equation}
which is still not convex due to the $\ell_0$-norm constraint.
To solve \eqref{eq:permute_LS_2}, we consider the following two approaches:
\begin{itemize}
  \item[({\uppercase\expandafter{\romannumeral 1}})] the iterative hard thresholding (IHT) approach \cite{blumensath2008iterative,jain2017non}, by working directly on the non-convex $\ell_0$-norm constraint, and using $H_s(\cdot)$ the hard thresholding operator to set the entries of a vector with small magnitude to zero and retain the large ones unaltered; or
  \item[({\uppercase\expandafter{\romannumeral 2}})] further relax the non-convex problem in \eqref{eq:permute_LS_2} by replacing the $\ell_0$-norm by $\ell_1$-norm, to get the following Lagrangian form,
  \begin{equation}\label{eq:permute_LS_3}
   \min_{\x \in \RR^{m^2},~\mathbf{d} \in \RR^{N^2} } \left\|{\rm vec}(\hat \bS_K) - \bTheta \x - \mathbf{d}  \right\|_2^2 + \gamma \| \mathbf{d} \|_1,
  \end{equation}
  for some hyperparameter $\gamma > 0$ that trade-offs the mean squared loss and the sparsity level in $\mathbf{d}$.
\end{itemize}

These two approaches allow for effectively solving the relaxed permuted linear regression in \eqref{eq:permute_LS_2} by alternately solving for $\mathbf{d}$ (via hard or soft thresholding) and $\x$ (via least squares, with sampling and/or sketching \cite{mahoney2011randomized,halko2011Finding} if necessary, see \Cref{rem:complexity} below).
The whole Kronecker \underline{\textbf{solving}} program is summarized in \Cref{alg:permuted_LR}.
In particular, note that when working with the convexly relaxed problem in \eqref{eq:permute_LS_3}, the convergence of this alternative minimization-type algorithm is always ensured, see for example \cite[Section~4.3]{jain2017non}.

\begin{algorithm}[tb]
   \caption{ Permuted linear regression to \underline{\textbf{solve}} for $\X$ }
   \label{alg:permuted_LR}
\begin{algorithmic}[1]
   \STATE {\bfseries Input:} Estimated $\hat \bS_K $ (from \Cref{alg:shrinkage_estim}, say), coefficient $\bTheta $ and hyperparameter $\gamma$ for convex relaxation or step length $\eta$ and sparsity level $s$ for IHT.
   \STATE {\bfseries Output:} Estimation of graph parameter $\hat \X$ by solving the permuted linear regression in  \eqref{eq:permute_LS_2}.
   \STATE{ Initialize $(\hat \x, \hat{\mathbf{d}} )$ }
   \WHILE {not converged}
   \OPTION {(\uppercase\expandafter{\romannumeral 1}) IHT}
    \STATE $ \hat {\mathbf{q}} \leftarrow (1- \eta) \hat {\mathbf{d}} + \eta ({\rm vec}(\hat \bS_K) - \bTheta \hat{\x})$;
    \STATE $\hat{\mathbf{d}} \leftarrow $ project $\hat{\mathbf{q}}$ onto the set of sparse vector via hard thresholding as $ \hat {\mathbf{d}} = H_s( \hat {\mathbf{q}})$;
    \ENDOPTION
    \OPTION {(\uppercase\expandafter{\romannumeral 2}) Convex relaxation}
   \STATE {$\hat{\mathbf{d}} \leftarrow \argmin_{ \mathbf{d} \in \RR^{N^2} } \|{\rm vec}(\hat \bS_K) - \bTheta \x - \mathbf{d} \|_2^2 + \gamma \| \mathbf{d} \|_1$ via soft thresholding};
   \ENDOPTION
   \STATE $ \hat{\mathbf{x}} \leftarrow ( \bTheta^\T \bTheta )^\dagger \bTheta^\T ({\rm vec}(\hat \bS_K) - \hat{\mathbf{d}} )$;
   \ENDWHILE
   \STATE \textbf{return} $\hat \X = {\rm mat}(\hat \x)$.
\end{algorithmic}
\end{algorithm}

\section{Numerical Evaluations and Discussions}
\label{sec:num}

The codes to reproduces the numerical results in this section are publicly available at \url{https://github.com/yqian108/Inference-of-Kronecker-Graph/}.

Before evaluating numerically the proposed algorithm, we first discuss the time complexity of Algorithms~\ref{alg:shrinkage_estim}~and~\ref{alg:permuted_LR} and ways to further reduce their running time using randomized numerical linear algebra (RNLA) techniques as follows.
\begin{Remark}[Time complexity of Algorithms~\ref{alg:shrinkage_estim}~and~\ref{alg:permuted_LR}]\normalfont
\label{rem:complexity}
For \Cref{alg:shrinkage_estim}, retrieving the few (of order at most $\log N$) singular values and vectors of a matrix $\A$ of size $N$ by $N$ with truncated SVD takes $\tilde O(N^2)$ (or $\tilde O({\rm nnz}(\A))$ for ${\rm nnz}(\A)$ the number of nonzero entries in $\A$, when $\A$ is sparse) time, see~\cite{baglama2005Augmented}.
This time cost can be further reduced to $\tilde O(N)$ using randomized SVD~\cite{halko2011Finding}.
For \Cref{alg:permuted_LR}, note that the coefficient matrix $\bTheta \in \RR^{N^2 \times m^2}$ can be decomposed into $N$ blocks as $\bTheta = [\bTheta_1;~\ldots;~\bTheta_N]$, with $\bTheta_i \in \RR^{N \times m}$ the coefficients corresponding to the $i$th vertex.
A direct implementation involving all $N$ blocks in $\bTheta$ take $O(N^2)$ time. 
By randomly sampling (an order $O(1)$ of) the $N$ blocks and solving the reduced problem, the time complexity can be further reduced to $O(N)$.
\Cref{fig:accelerated} below presents numerical results on Algorithms~\ref{alg:shrinkage_estim}~and~\ref{alg:permuted_LR} with RNLA acceleration, for which no performance drop is observed.
\end{Remark}

\subsection{Evaluations on random Kronecker graphs}
\label{subsec:synthetic_Kron}

We compare, in \Cref{fig:compare_Kronfit_IHT}, the performance and running time of the proposed denoise-and-solve inference method in \Cref{algo:meta} (with both IHT and convex relaxation approaches in \Cref{alg:permuted_LR}) to that of the KronFit algorithm proposed in \cite{leskovec2010kronecker}, on random Kronecker graphs generated according to \Cref{def:random_Kronecker_graph}.

We observe from \Cref{fig:compare_Kronfit_IHT} that, in terms of performance, the proposed approaches marginally fall short compared to KronFit for \emph{very sparse} graphs (with an average connection probability down to $10^{-6}$) but outperforms KronFit for \emph{slightly denser} Kronecker graphs. 
The performance of the proposed \Cref{algo:meta} gets better as the graph becomes denser: This is not surprising, since it is designed for denser graphs.
In terms of time complexity, we see from~\Cref{fig:compare_Kronfit_IHT} that the running time of KronFit grows rapidly as the graph gets denser, while \Cref{algo:meta} consistently maintains a commendably low running time. 

\begin{Remark}[Numerical stability]\normalfont
\label{rem:num_stable}
Also note that for a given graph, the performance of \Cref{algo:meta} is deterministic. 
This is in contrast to KronFit, which exhibits significant variance in its performance, as a consequence of its reliance on sampling (to solve the vertices matching problem, see again \Cref{rem:correspondence}). 
This numerical instability of KronFit is further confirmed on \Cref{table::classification} below and limits its applications in downstream tasks such as graph classification.  
\end{Remark}

Turning our attention to the comparative analysis between the IHT and convex relaxation methods inside \Cref{alg:permuted_LR}.
We observe in \Cref{fig:compare_Kronfit_IHT} that these two approaches demonstrate analogous efficacy and computational efficiency, for not-so-sparse graphs. 
This observation aligns with recent trend of \emph{non-convex optimization} in ML, that advises \emph{not} to relax the non-convex problems but to solve them directly (e.g., with the non-convex IHT). 
While seemingly doomed to fail, this approach is shown to work well, both theoretically and empirically, in a series of illuminating results, if the problem has nice structure~\cite{jain2017non}.
We conjecture that such property also holds for random Kronecker graphs.

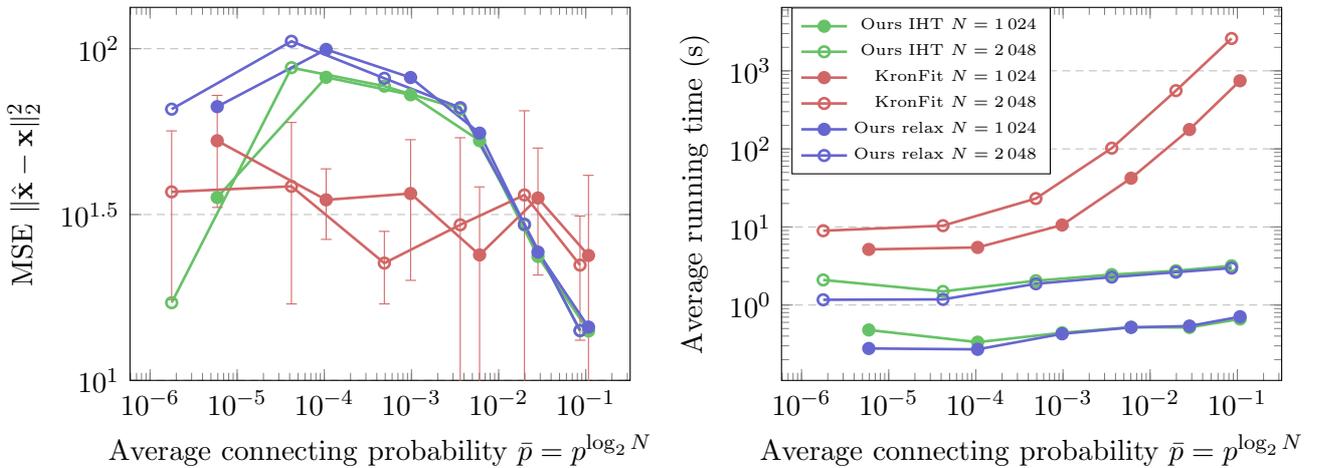
\begin{figure}[thb!]
\centering
\begin{minipage}[c]{0.48\textwidth}
\begin{tikzpicture}
\renewcommand{\axisdefaulttryminticks}{4} 
\pgfplotsset{every major grid/.style={densely dashed}}       
\tikzstyle{every axis y label}+=[yshift=-10pt] 
\tikzstyle{every axis x label}+=[yshift=5pt]
\pgfplotsset{every axis legend/.append style={cells={anchor=east},fill=white, at={(0.02,0.98)}, anchor=north west, font=\tiny}}
\begin{axis}[
width=\textwidth,
height=.8\textwidth,
ymin = 10,
ymajorgrids=true,
scaled ticks=true,
xlabel = { Average connecting probability $\bar p = p^{\log_2 N}$ },
ylabel = { MSE $\| \hat \x - \x \|_2^2$ },
scaled ticks=true,
ymode=log,
xmode=log
]
\addplot[mark=*,color=GREEN!60!white,line width=1pt] coordinates{
(0.300000^10,  35.6052)(0.400000^10, 81.9956)(0.500000^10,72.6274)(0.600000^10, 52.8076 )(0.700000^10,23.6445 )(0.800000^10,14.0928)
};
\addplot[mark=o,color=GREEN!60!white,line width=1pt] coordinates{
(0.300000^11, 17.1496)(0.400000^11, 87.6384)(0.500000^11,77.1122)(0.600000^11,  65.7416  )(0.700000^11,29.4007 )(0.800000^11,14.1258)
};
\addplot[
mark=*,color=RED!60!white,line width=1pt,
error bars/.cd,
y dir=both, y explicit,
] 
coordinates{
(0.300000^10,52.7806) +- (0,19.5609)
(0.400000^10,35.0183) +- (0,8.4080)
(0.500000^10,36.5956) +- (0,16.5740)
(0.600000^10,23.8912 ) +- (0,14.4315)
(0.700000^10,35.4720) +- (0,14.6983)
(0.800000^10,23.8048) +- (0,17.7007)
};
\addplot[
mark=o,color=RED!60!white,line width=1pt,
error bars/.cd,
y dir=both, y explicit,
] 
coordinates{
(0.300000^11, 37.0184) +- (0,19.5041)
(0.400000^11,38.4619) +- (0,21.4621)
(0.500000^11, 22.5698) +- (0, 5.5682)
(0.600000^11,  29.4469) +- (0,24.4514)
(0.700000^11, 36.2220) +- (0,28.7647)
(0.800000^11,22.2626) +- (0, 9.0481)
};
\addplot[mark=*,color=BLUE!60!white,line width=1pt] coordinates{
(0.300000^10,66.882025)(0.400000^10,99.492925)(0.500000^10,81.862078)(0.600000^10,55.635160)(0.700000^10,24.373773)(0.800000^10,14.481985)
};
\addplot[mark=o,color=BLUE!60!white,line width=1pt] coordinates{
(0.300000^11,65.721723)(0.400000^11,105.289800)(0.500000^11,81.415950)(0.600000^11,66.411783)(0.700000^11,29.563135)(0.800000^11,14.132318)
};
\end{axis}
\end{tikzpicture}
  \end{minipage}
  \hfill{}%
  \begin{minipage}[c]{0.48\textwidth}
\begin{tikzpicture}
\renewcommand{\axisdefaulttryminticks}{4} 
\pgfplotsset{every major grid/.style={densely dashed}}       
\tikzstyle{every axis y label}+=[yshift=-10pt] 
\tikzstyle{every axis x label}+=[yshift=5pt]
\pgfplotsset{every axis legend/.append style={cells={anchor=east},fill=white, at={(0.02,1.0)}, anchor=north west, font=\tiny }}
\begin{axis}[
width=\columnwidth,
height=.8\textwidth,
xlabel = { Average connecting probability $\bar p = p^{\log_2 N}$ },
ymode=log,
xmode=log,
ylabel = { Average running time (s) },
ymajorgrids=true,
scaled ticks=true,
]
\addplot[mark=*,color=GREEN!60!white,line width=1pt] coordinates{
(0.300000^10,0.4792)(0.400000^10,0.3345)(0.500000^10,0.4411)(0.600000^10,0.5202)(0.700000^10,0.5158)(0.800000^10,0.6624)
};

\addlegendentry{{ Ours IHT $N = 1\,024$ }};
\addplot[mark=o,color=GREEN!60!white,line width=1pt] coordinates{
(0.300000^11,2.0967)(0.400000^11, 1.4942)(0.500000^11,2.0474)(0.600000^11,2.4527)(0.700000^11,2.7461)(0.800000^11,3.1809)
};
\addlegendentry{{ Ours IHT $N = 2\,048$ }};
\addplot[mark=*,color=RED!60!white,line width=1pt] coordinates{
(0.300000^10,5.158000)(0.400000^10,5.466900)(0.500000^10,10.582500)(0.600000^10,42.134200)(0.700000^10,177.084600)(0.800000^10,746.147300)
};
\addlegendentry{{ KronFit $N = 1\,024$ }};
\addplot[mark=o,color=RED!60!white,line width=1pt] coordinates{
(0.300000^11,8.943200)(0.400000^11,10.383900)(0.500000^11,23.129700)(0.600000^11,102.2133)(0.700000^11,557.5088)(0.800000^11,2594.729200)
};
\addlegendentry{{ KronFit $N = 2\,048$ }};
\addplot[mark=*,color=BLUE!60!white,line width=1pt] coordinates{
(0.300000^10, 0.277700)(0.400000^10,0.270700)(0.500000^10,0.427600)(0.600000^10, 0.517100 )(0.700000^10,0.538100)(0.800000^10,0.708700)
};
\addlegendentry{{ Ours relax $N = 1\,024$ }};
\addplot[mark=o,color=BLUE!60!white,line width=1pt] coordinates{
(0.300000^11,1.167600)(0.400000^11, 1.181100 )(0.500000^11, 1.872400)(0.600000^11,2.277600)(0.700000^11,2.632800)(0.800000^11,2.973400)
};
\addlegendentry{{ Ours relax $N = 2\,048$ }};
\end{axis}
\end{tikzpicture}
\end{minipage}
\caption{{Estimation MSEs (\textbf{left}) and running time (\textbf{right}) of KronFit versus \Cref{algo:meta} (with IHT and convex relaxation), on random Kronecker graphs in \Cref{def:random_Kronecker_graph}, with $p\in [0.3, 0.8]$, $\x = [5.25, 0.25, 2.25, -7.75]$, and $20\%$ vertices randomly shuffled, for $N = 1\,024, 2\,048$, $s = 5$ in \Cref{alg:permuted_LR} for IHT. Result obtained over $10$ independent runs on the \emph{same} graph. }}
\label{fig:compare_Kronfit_IHT}
\end{figure}




We then evaluate, in \Cref{fig:stability_and_shuffle_ratios}, the performance of \Cref{algo:meta} as a function of the ratio of node perturbation.
We observe that as the number of node perturbation increases, both IHT and convex relaxation approaches exhibit a diminishing level of accuracy.
This observation is in line with our strategic departure from the original permuted linear regression in \eqref{eq:permute_LS_1}, to the relaxed formulation in \eqref{eq:permute_LS_2} that hypothesizes a sparsely permuted structure.

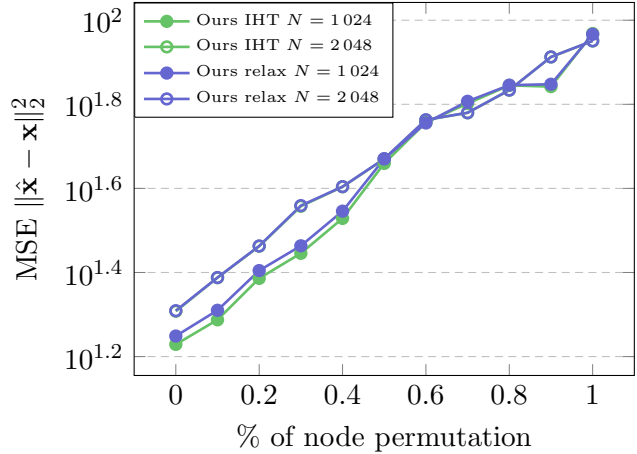
\begin{figure}[thb]
  \centering
  \begin{minipage}[c]{0.48\textwidth}
\begin{tikzpicture}
\renewcommand{\axisdefaulttryminticks}{4} 
\pgfplotsset{every major grid/.style={densely dashed}}       
\tikzstyle{every axis y label}+=[yshift=-10pt] 
\tikzstyle{every axis x label}+=[yshift=5pt]
\pgfplotsset{every axis legend/.append style={cells={anchor=west},fill=white, at={(0,1)}, anchor=north west, font=\tiny }}
\begin{axis}[
width=\columnwidth,
height=.8\columnwidth,
ymajorgrids=true,
scaled ticks=true,
xlabel = { $\%$ of node permutation },
ylabel = { MSE $\| \hat \x - \x \|_2^2$ },
scaled ticks=true,
ymode=log
]
\addplot[mark=*,color=GREEN!60!white,line width=1pt] coordinates{
(0.000000,16.9422)
(0.100000,19.3862)
(0.200000, 24.3231)
(0.300000,27.8984)
(0.400000,33.7782)
(0.500000, 45.6503)
(0.600000,57.1454)
(0.700000, 63.3714)
(0.800000,69.9851)
(0.900000,69.5102)
(1.000000,92.9465)
};
\addlegendentry{{ Ours IHT $N = 1\,024$ }}; 
\addplot[mark=o,color=GREEN!60!white,line width=1pt] coordinates{
(0.000000,20.3161)
(0.100000,24.4114)
(0.200000,29.0018)
(0.300000, 36.1022)
(0.400000,40.1582)
(0.500000,46.7567)
(0.600000, 57.9052)
(0.700000,60.2324)
(0.800000,68.1716)
(0.900000,81.7472)
(1.000000, 89.2492)
};
\addlegendentry{{ Ours IHT $N = 2\,048$ }}; 
\addplot[mark=*,color=BLUE!60!white,line width=1pt] coordinates{
(0.000000,17.736)
(0.100000,20.4203)
(0.200000,25.38)
(0.300000,29.0678)
(0.400000,35.1445)
(0.500000,46.817)
(0.600000,56.962)
(0.700000,64.1397)
(0.800000,70.0687)
(0.900000,70.4211)
(1.000000,92.5763)
};
\addlegendentry{{ Ours relax $N = 1\,024$ }}; 
\addplot[mark=o,color=BLUE!60!white,line width=1pt] coordinates{
(0.000000,20.3758)
(0.100000,24.4314)
(0.200000,29.043)
(0.300000,36.2337)
(0.400000,40.1913)
(0.500000,46.8165)
(0.600000,57.9653)
(0.700000,60.2415)
(0.800000,68.2045)
(0.900000,81.7779)
(1.000000,89.3011)
};
\addlegendentry{{ Ours relax $N = 2\,048$ }}; 
\end{axis}
\end{tikzpicture}
\end{minipage}
\caption{{
Estimation MSEs of \Cref{algo:meta} on Kronecker graphs as in \Cref{def:random_Kronecker_graph} with $p = 0.7$, $N = 1\,024, 2\,048$, and same $\x$ as in \Cref{fig:compare_Kronfit_IHT}, as a function of the percentage of node permutation.}}
\label{fig:stability_and_shuffle_ratios}
\end{figure}

We further assess, in \Cref{fig:accelerated}, the use of RNLA techniques (e.g, randomized SVD and random sampling in \Cref{rem:complexity}) to reduce the time complexity of \Cref{algo:meta}. 
The results in \Cref{fig:accelerated} show that the incorporation of RNLA techniques yields a significant reduction in running time (theoretically scaling down from $O(N^2)$ to $O(N)$). 
Impressively, this accelerated framework demonstrates virtually no compromise in performance, particularly for not-so-sparse Kronecker graphs. 
This outcome underscores the efficacy of RNLA techniques in further enhancing the computational efficiency without sacrificing the overall effectiveness of \Cref{algo:meta}.



\begin{figure}[h!]
\centering
\begin{minipage}[c]{0.48\textwidth}
\begin{tikzpicture}
\renewcommand{\axisdefaulttryminticks}{4} 
\pgfplotsset{every major grid/.style={densely dashed}}       
\tikzstyle{every axis y label}+=[yshift=-10pt] 
\tikzstyle{every axis x label}+=[yshift=5pt]
\pgfplotsset{every axis legend/.append style={cells={anchor=east},fill=white, at={(0.02,0.02)}, anchor=south west, font=\tiny}}
\begin{axis}[
width=\textwidth,
height=.8\textwidth,
ymin = 10,
ymajorgrids=true,
scaled ticks=true,
xlabel = { Average connecting probability $\bar p = p^{\log_2 N}$ },
ylabel = { MSE $\| \hat \x - \x \|_2^2$ },
scaled ticks=true,
ymode=log,
xmode=log
]
\addplot[
mark=triangle,color=GREEN!60!white,line width=1pt,
error bars/.cd,
y dir=both, y explicit,
] 
coordinates{
(0.300000^11, 92.9956) +- (0,0.62966)  
(0.400000^11, 85.0374) +- (0, 14.8124)
(0.500000^11,  85.1517) +- (0,5.3754)
(0.600000^11, 70.8328 ) +- (0,3.722)
(0.700000^11, 32.0473) +- (0,2.8556)
(0.800000^11,  13.501 ) +- (0,1.4383)
};
\addlegendentry{{ Ours IHT, accelerated }};
\addplot[
mark=triangle,color=BLUE!60!white,line width=1pt,
error bars/.cd,
y dir=both, y explicit,
] 
coordinates{
(0.300000^11, 99.7884) +- (0,14.912)
(0.400000^11, 93.2702) +- (0,18.7938)  
(0.500000^11,  81.0844 ) +- (0,7.0712)
(0.600000^11,82.6631) +- (0,3.5008)
(0.700000^11, 45.4322) +- (0,2.4886)
(0.800000^11,14.7817 ) +- (0,1.7803)
};
\addlegendentry{{ Ours relax, accelerated }};
\addplot[mark=o,color=GREEN!60!white,line width=1pt] coordinates{
(0.300000^11, 17.1496)(0.400000^11, 87.6384)(0.500000^11,77.1122)(0.600000^11,  65.7416  )(0.700000^11,29.4007 )(0.800000^11,14.1258)
};
\addlegendentry{{ Ours IHT }};
\addplot[mark=o,color=BLUE!60!white,line width=1pt] coordinates{
(0.300000^11,65.721723)(0.400000^11,105.289800)(0.500000^11,81.415950)(0.600000^11,66.411783)(0.700000^11,29.563135)(0.800000^11,14.132318)
};
\addlegendentry{{ Ours relax }};
\end{axis}
\end{tikzpicture}
  \end{minipage}
  \hfill{}%
  \begin{minipage}[c]{0.48\textwidth}
\begin{tikzpicture}
\renewcommand{\axisdefaulttryminticks}{4} 
\pgfplotsset{every major grid/.style={densely dashed}}       
\tikzstyle{every axis y label}+=[yshift=-10pt] 
\tikzstyle{every axis x label}+=[yshift=5pt]
\pgfplotsset{every axis legend/.append style={cells={anchor=east},fill=white, at={(0.02,1.0)}, anchor=north west, font=\tiny }}
\begin{axis}[
width=\columnwidth,
height=.8\textwidth,
xlabel = { Average connecting probability $\bar p = p^{\log_2 N}$ },
ymode=log,
xmode=log,
ylabel = { Average running time (s) },
ymajorgrids=true,
scaled ticks=true,
]
\addplot[mark=triangle,color=GREEN!60!white,line width=1pt] coordinates{
(0.300000^11,0.6979)(0.400000^11, 0.6715)(0.500000^11, 0.6603)(0.600000^11,0.6458)(0.700000^11,0.6328)(0.800000^11,0.5610)
};
\addplot[mark=triangle,color=BLUE!60!white,line width=1pt] coordinates{
(0.300000^11,0.5824)(0.400000^11, 0.6042)(0.500000^11, 0.5871)(0.600000^11,0.5944)(0.700000^11,0.5054)(0.800000^11,0.4692)
};
\addplot[mark=o,color=BLUE!60!white,line width=1pt] coordinates{
(0.300000^11,1.167600)(0.400000^11, 1.181100)(0.500000^11, 1.872400)(0.600000^11,2.277600)(0.700000^11,2.632800)(0.800000^11,2.973400)
};
\addplot[mark=o,color=GREEN!60!white,line width=1pt] coordinates{
(0.300000^11,2.0967)(0.400000^11, 1.4942)(0.500000^11,2.0474)(0.600000^11,2.4527)(0.700000^11,2.7461)(0.800000^11,3.1809)
};
\end{axis}
\end{tikzpicture}
\end{minipage}
\caption{ Estimation MSEs (\textbf{left}) and running time (\textbf{right}) of \Cref{algo:meta}, with and without RNLA acceleration in \Cref{rem:complexity}, on random Kronecker graphs as in \Cref{fig:compare_Kronfit_IHT} for $N = 2\,048$.
For randomized SVD~\cite{halko2011Finding}, we use an iteration count of $q=2$; for random sampling, we choose $100$ from $N$ blocks uniformly at random.
Result obtained over $10$ independent runs. }
\label{fig:accelerated}
\end{figure}
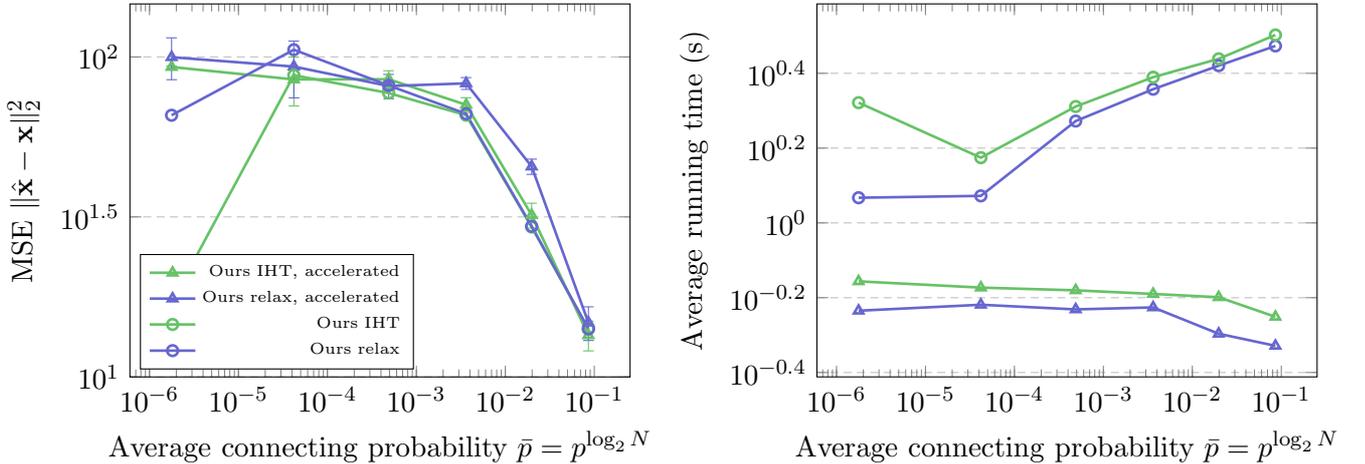


Similar conclusion can be reached when comparing \Cref{algo:meta} to the moment-based approach proposed in~\cite{gleich2012MomentBased}, and we refer the readers to \Cref{sec:additional_num} for additional numerical results on Kronecker graph inference.

\begin{table*}[htb!] 
  \caption{ Accuracy of graph classification using \Cref{algo:meta}, KronFit, and different GNN-based methods: GCN baseline in~\cite{errica2019Fair}, DGCNN~\cite{zhang2018end}, DiffPool~\cite{ying2018hierarchical}, and ECC~\cite{simonovsky2017dynamic}. For REDDIT-B, IMDB-B, COLLAB, IMDB-M and REDDIT-5K datasets, the two accuracies represent whether graph node degrees are used as additional input features. The training of ECC on REDDIT-B, COLLAB and REDDIT-5K are out of time ($> 72$ hours for a single training) and omitted.  } 
  \label{table::classification}
  \centering
  \begin{tabular}{lcccccccc}
    \toprule
    Datasets & Ours & KronFit & GCN baseline & DGCNN & DiffPool & ECC \\
    \midrule
    PROTEINS & $72.0\% \pms{3.3}$ & $65.0\% \pms{2.7}$ &  $75.8 \% \pms{3.7}$ &$72.9 \% \pms{3.5}$&$73.7 \% \pms{3.5}$& $72.3 \% \pms{3.4}$     \\
    NCI1 &$60.4\% \pms{1.9}$ & $57.6\% \pms{1.9}$ & $69.8 \% \pms{2.2}$ & $76.4\% \pms{1.7}$ & $76.9\% \pms{1.9}$ & $76.2 \% \pms{1.4}$  \\
    REDDIT-B & $81.4 \% \pms{3.0}$ & $72.7\% \pms{2.8}$ & $72.1 \% \pms{7.8}$ & $77.1\% \pms{2.9} $ & $76.6\% \pms{2.4}$ & {\sf out of time} \\
        & $80.5 \% \pms{1.7}$ & $73.1\% \pms{3.0}$ & $82.2\% \pms{3.0}$ & $87.8\% \pms{2.5}$  & $89.1\% \pms{1.6}$ & {\sf out of time} \\
    
    IMDB-B & $66.0 \% \pms{2.7}$ & $60.9\% \pms{ 3.4}$ & $50.7 \% \pms{ 2.4}$ & $53.3\% \pms{ 5.0}$ &  $68.3\% \pms{ 6.1}$ & $67.8 \% \pms{ 4.8}$ \\
          & $65.3\% \pms{ 4.6}$ & $60.0\% \pms{ 4.3}$ & $70.8\% \pms{ 5.0}$ & $69.2\% \pms{ 3.0}$ & $68.4\% \pms{ 3.3}$ & $67.7 \% \pms{ 2.8}$ \\
    ENZYMES & $35.6 \% \pms{ 7.8}$& $20.6 \% \pms{ 4.8}$ & $65.2 \% \pms{ 6.4}$ & $38.9 \% \pms{ 5.7}$ & $59.5 \% \pms{ 5.6}$ & $29.5 \% \pms{ 8.2}$\\
    COLLAB & $68.0 \% \pms{ 1.7}$ & $62.4 \% \pms{ 1.8}$ &$55.0 \% \pms{ 1.9}$& $57.4 \% \pms{ 1.9}$& $67.7 \% \pms{ 1.9}$& {\sf out of time} \\
        & $68.3 \% \pms{ 1.8}$ &$63.0 \% \pms{ 1.9}$ &$70.2 \% \pms{ 1.5}$& $71.2 \% \pms{ 1.9}$& $68.9 \% \pms{ 2.0}$ &  {\sf out of time}\\
    IMDB-M & $46.7 \% \pms{ 3.2}$ & $41.3 \% \pms{ 4.8}$ &$36.1 \% \pms{ 3.0}$ & $38.6 \% \pms{ 2.2}$& $45.1 \% \pms{ 3.2}$ & $44.8 \% \pms{ 3.1}$\\
         & $46.4 \% \pms{ 4.2}$& $42.5 \% \pms{ 3.7}$& $49.1 \% \pms{ 3.5}$& $45.6 \% \pms{ 3.4}$& $45.6 \% \pms{ 3.4}$ & $43.5 \% \pms{ 3.1}$\\
    REDDIT-5K & $42.8 \% \pms{ 1.4}$& $41.2 \% \pms{ 1.5}$ &$35.1 \% \pms{ 1.4}$ &$35.7 \% \pms{ 1.8}$ &$34.6 \% \pms{ 2.0}$ & {\sf out of time}\\
          & $43.3 \% \pms{ 1.8}$& $40.7 \% \pms{ 2.1}$ &$52.2 \% \pms{ 1.5}$ & $49.2 \% \pms{ 1.2}$& $53.8 \% \pms{ 1.4}$& {\sf out of time}\\
    \bottomrule
  \end{tabular}
\end{table*}

\subsection{Application to realistic graph classification}
\label{subsec:application_realistic}

We have conducted experiments in \Cref{subsec:synthetic_Kron} showing the effectiveness of the proposed \Cref{algo:meta}, in the inference of random Kronecker graphs.
In the following, we consider the use of Kronecker graph model and~\Cref{algo:meta}, as feature extractors for large-scale realistic graphs.

We focus on the task of (binary and multi-class) graph classification, on a range of chemical and social graphs  as in \cite{errica2019Fair}. 
See \Cref{sec:datasets} for the statistics for these datasets.
We compare, in \Cref{table::classification}, the performance of\footnote{ For \Cref{algo:meta}, we choose $m=5$ for PROTEINS and ENZYMES datasets, and $m=4$ otherwise. The graph features are standardized before classification, and we follow the pre-computed data partitions as in \cite{errica2019Fair}. }
\begin{itemize}
    \item[({\romannumeral 1})] Kronecker graph inference approaches of the proposed \Cref{algo:meta} and KronFit~\cite{leskovec2010kronecker} as graph feature extractors, followed by a single-layer MLP with ReLU activation; \emph{versus}
    \item[({\romannumeral 2})] a few popular baselines based on graph neural networks (GNNs) include: GCN baseline in~\cite{errica2019Fair}, DGCNN~\cite{zhang2018end}, DiffPool~\cite{ying2018hierarchical}, and ECC~\cite{simonovsky2017dynamic}.
\end{itemize}
Note from \Cref{table::classification} that \Cref{algo:meta}, by efficiently exploiting the graph topological information, consistently outperforms KronFit, and achieves comparable performance to popular GNN baselines on a variety of realistic graph datasets. 
This suggests that the features obtained from \Cref{algo:meta} can be used as effective representations for realistic graphs.

It is worth noting from \Cref{table::classification} that on the chemical ENZYMES dataset~\cite{schomburg2004brenda}, both the proposed \Cref{algo:meta} and KronFit exhibit a notable decrease in accuracy. 
This may be attributed to the dataset's relatively limited size and higher number of classes, for which stronger feature extractors are needed.
For social graphs such as REDDIT-B, IMDB-B, COLLAB, IMDB-M, and REDDIT-5K~\cite{yanardag2015deep}, the experiments in \cite{errica2019Fair} are conducted with and without the node degrees as the input features.
Interestingly, note that when compared to the GCN baseline in~\cite{errica2019Fair}, our \Cref{algo:meta} and KronFit consistently outperform the GCN baselines on \emph{all} aforementioned social graph datasets \emph{in the absence of} the node degree features, but fail when node degree features are present.
This numerical evidence seemingly suggests that Kronecker graph model can implicitly explores the node degree features.

\section{Conclusion}
\label{sec:con}

In this paper, we investigate the large-dimensional behavior of random Kronecker graphs
We show that the graph adjacency $\A$ is close, in spectral norm, to the sum of a small-rank signal $\bS_K$ and a random noise matrix $\Z$.
Based on this observation, we propose a ``denoise-and-solve'' \Cref{algo:meta} for graph parameters inference, and discuss its practical implementation.
Numerical experiments are provided to validate the effectiveness (in terms of performance and running time) of the proposed approach against the KronFit method, on random Kronecker graphs. 
We further propose to use Kronecker graph model and \Cref{algo:meta} as features extractors for realistic graph classification, and provide comparative analysis to a few popular GNN baselines.

\subsubsection*{Acknowledgments}
Z.~Liao would like to acknowledge the National Natural Science Foundation of China (via fund NSFC-62206101 and NSFC-12141107), the Fundamental Research Funds for the Central Universities of China (2021XXJS110), the Key Research and Development Program of Guangxi (GuiKe-AB21196034) for providing partial support.

Y.~Xiao was supported in part by the National Natural Science Foundation of China under grant 62071193, the Key R \& D Program of Hubei Province of China under grants 2021EHB015 and 2020BAA002, and the major key project of Peng Cheng Laboratory (No.\@ PCL2021A12).

\printbibliography


\clearpage


\appendix

The supplementary material is organized as follows: 
A few useful lemmas that will be consistently exploited in the proof are listed in \Cref{sec:SM-lemmas}.
The proofs of the technical results in the paper are given in \Cref{sec:appendix_proof}.
We provide in \Cref{sec:spectral_analysis_centered_adjacency} some detailed results on the spectral analysis of large Kronecker graphs that may be of independent interest.
Additional numerical results are given in \Cref{sec:additional_num}.
The statistics of the graph classification datasets used in \Cref{subsec:application_realistic} are reported in \Cref{sec:datasets}.

\section{Useful Lemmas}\label{sec:SM-lemmas}

Here we list a few lemmas that will be constantly used in the proof.

\begin{Lemma}[Weyl's inequality, {\cite[Theorem~4.3.1]{horn2012matrix}}]\label{lem:weyl}
Let $\A,\B \in \RR^{p \times p}$ be symmetric matrices and let the respective eigenvalues of $\A$, $\B$ and $\A+\B$ be arranged in decreasing order, i.e., $\lambda_1 \geq \lambda_2 \geq \ldots \geq \lambda_p$. Then, for all $i\in\{1,\ldots,p\}$,
\begin{equation}
  \lambda_{i+j-1}(\A) + \lambda_{p+1-j} (\B) \leq \lambda_i(\A + \B) \leq \lambda_{i-j} (\A) + \lambda_{j+1} (\B)
\end{equation}
In particular,
\begin{align*}
  \max_{1\leq i\leq p} |\lambda_i(\A) - \lambda_i(\B)| \leq \| \A - \B \|_2.
\end{align*}
\end{Lemma}

\begin{Lemma}[SVD of Kronecker product, {\cite[Theorem~4.2.15]{horn1991topics}}]\label{lem:svd_kron}
Let $\A \in \RR^{p \times q}, \B \in \RR^{m \times n}$ having rank $r_{\A}, r_{\B}$, and let $ \A =  \U_{\A} \bSigma_{\A} \V_{\A}^\T$ and $ \B = \U_{\B} \bSigma_{\B} \V_{\B}^\T$ be the singular value decomposition of $\A$ and $\B$, respectively. 
Then, the singular value decomposition of the Kronecker product $\A \otimes \B$ is given by
\begin{equation}
    \A \otimes \B = (\U_{\A} \bSigma_{\A} \V_{\A}^\T) \otimes (\U_{\B} \bSigma_{\B} \V_{\B}^\T) = (\U_{\A} \otimes \U_{\B})(\bSigma_{\A} \otimes \bSigma_{\B})(\V_{\A} \otimes \V_{\B})^\T,
\end{equation}
with $\rank (\A \otimes \B) = \rank (\B \otimes \A) = r_{\A} \cdot r_{\B}$ and $ \| \A \otimes \B \|_2 = \| \A \|_2\cdot \| \B \|_2$.
\end{Lemma}

\begin{Lemma}
\label{lem:vec_Kron}
For matrices $\A \in \RR^{m \times n}$, $\B \in \RR^{n \times p}$, and $\C \in \RR^{p \times q}$, we have
\begin{equation}
   {\rm vec}(\A \B \C) = (\C^\T \otimes \A) {\rm vec} (\B).
\end{equation}
\end{Lemma}

\section{Mathematical Proofs}
\label{sec:appendix_proof}

\subsection{ Proof of \Cref{prop:linearized} }
\label{sec:appendix_proof_of_linearized_rank}

In the section, we present the proof of the three items (i.e, Item~({\romannumeral 1}),~({\romannumeral 2}),~and~({\romannumeral 3})) of \Cref{prop:linearized} in \Cref{subsec:SM_proof_Item_i}, \Cref{subsec:SM_proof_Item_ii}, and \Cref{subsec:SM_proof_Item_iii}, respectively.

\subsubsection{Proof of \Cref{prop:linearized} Item~(i)}
\label{subsec:SM_proof_Item_i}

First note, under Assumption~\ref{ass:growth-rate} and by the Binomial theorem that,
\begin{align}
    &[\bP_k]_{ij} = \prod_{u,v=1}^m P_{uv}^{C_{uv; ij}} = \prod_{u,v=1}^m ( p + X_{uv}/\sqrt N)^{C_{uv; ij}} \nonumber \\ 
    &= \prod_{u,v=1}^m \left( p^{ C_{uv;ij} } + C_{uv;ij} \cdot p^{ (C_{uv;ij} -1) } X_{uv}/ \sqrt N + \tilde O(N^{-1}) \right) \nonumber \\ 
    &= p^k + p^{ (k-1) } \sum_{u,v=1}^m C_{uv;ij} \cdot X_{uv}/\sqrt N + \tilde O(N^{-1}), \label{eq:approx_bP_K}
\end{align}
for $1 \leq k \leq K$ with coefficients $C_{uv; ij} \in \mathbb N$ satisfying $\sum_{u,v=1}^m C_{uv; ij} = k$.
This gives, by the definition of the Kronecker power, that in matrix form,
\begin{align*}
    \bP_{k} &= \bP_{k-1} \otimes \bP_1 \\
    &= (p^{k-1}\one_{m^{k-1}}\one_{m^{K-1}}^\T + \tilde \bS_{{k-1}} + \tilde O_{ \| \cdot \|_2 }(1) ) \otimes(p \one_m \one_m^\T + \X/\sqrt N) \\
    &= p^k\one_{m^k}\one_{m^k}^\T + p^{k-1}\one_{m^{k-1}}\one_{m^{k-1}}^\T \otimes \X/\sqrt N + p \tilde \bS_{{k-1}} \otimes (\one_m\one_m^\T) + \tilde O_{ \| \cdot \|_2 }(1),
\end{align*}
where we used in the second line $\tilde \bS_{{k-1}}$ to denote terms in \eqref{eq:approx_bP_K} that are linear in (the entries of) $\X$ with $\tilde \bS_1 = \X/\sqrt N$, and $\tilde O_{ \| \cdot \|_2 }(1)$ to denote matrices of spectral norm order $\tilde O(1)$, as a consequence of the fact that $\| \A \|_2 \leq N \| \A \|_{\max}$ for $\A \in \RR^{N \times N}$ and $\| \A \|_{\max} \equiv \max_{i,j} |A_{ij}|$; and in third line the fact that $\| \tilde \bS_{k-1} \|_2 = \tilde O(\sqrt N)$ and Lemma~\ref{lem:svd_kron} so that $\tilde \bS_{k-1} \otimes \X/\sqrt N = \tilde O_{\| \cdot \|_2}(1)$.

Note that by definition of $\bS_k$ in \eqref{eq:def_S}, $\bP_k$ can be rewritten as the following recursion on $\bS_k$,
\begin{equation}
  \bP_k = p^k\one_{m^k}\one_{m^k}^\T + \sqrt N \bS_k + \tilde O_{ \| \cdot \|_2 }(1), \quad \bS_1 = \X/N,
\end{equation}
by taking $\bS_k = \tilde \bS_k/\sqrt N$ for $1 \leq k \leq K$, and by taking $k=K$, one has
\begin{equation}
  \bP_K = p^K\one_{N}\one_{N}^\T + \sqrt N \bS_K + \tilde O_{ \| \cdot \|_2 }(1) \equiv \bP^{\rm lin}_{K} + \tilde O_{ \| \cdot \|_2 }(1), \quad \bS_1 = \X/N,
\end{equation}
where we introduced the \emph{linearized} probability matrix $\bP^{\rm lin}_{K} \in \RR^{N \times N}$ as in \eqref{eq:def_P_lin}.
This concludes the proof of Item~(i).

\subsubsection{Proof of \Cref{prop:linearized} Item~(ii)}
\label{subsec:SM_proof_Item_ii}

For Item~(ii), note from the recursive definition in \eqref{eq:def_S} that
 \begin{align*}
  \bS_k  &= \frac{p^{k-1}}N \one_{m^{k-1}}\one_{m^{k-1}}^\T \otimes \X + p\bS_{k-1}\otimes \one_{m}\one_{m}^\T  \\
      &= \frac{p^{k-1}}N  \underbrace{(\one_m \one_m^\T\otimes\one_m \one_m^\T\otimes\cdots\otimes \X
      +\cdots+\X\otimes\cdots\otimes\one_m \one_m^\T\otimes\one_m \one_m^\T )}_{k~\text{times}}.
\end{align*}
Since each term in the bracket is a linear combination of ${\rm vec} (\X)$, with coefficients of the type $p^{k-1} C_{uv; ij}/N  = \tilde O(N^{-1})$ as in \eqref{eq:approx_bP_K}. 
Taking $k = K$ allows one to conclude that $\bS_K = {\rm mat} (\bTheta {\rm vec}(\X))$ for some coefficients $\bTheta \in \RR^{N^2 \times m^2}$ with $\| \bTheta \|_{\max} = \tilde O(N^{-1})$.

\medskip

In the following, we will show that $\bTheta \one_{m^2} = \frac{p^{K-1} K}N \one_{N^2}, \bTheta^\T \one_{N^2} = \frac{p^{K-1}KN}{m^2} \one_{m^2}$.
Note that the former follows straightforwardly from the binomial expansion in \eqref{eq:approx_bP_K} that $\sum_{u,v=1}^m C_{uv; ij} = K$ after $K$-th Kronecker product.

We now prove $\bTheta^\T \one_{N^2} = \frac{p^{K-1}KN}{m^2} \one_{m^2}$ by exploiting the structures in the columns of $\bTheta$.
First, for $1 \leq k \leq K$, we have
\begin{equation}
    \bS_k  = {\rm mat}(\bTheta_k {\rm vec}(\X)), 
\end{equation}
for some coefficients $\bTheta_k \in \RR^{m^{2k} \times m^2}$ determined by the binomial expansion and $p,k$. 
With a slight abuse of notations, we denote $\bTheta \equiv \bTheta_K \in \RR^{N^2 \times m^2}$.
Denote $\btheta_{i,k} \in \RR^{m^{2k}}$ the $i$-th column of $\bTheta_k$, we have, for $1 \leq i, j \leq m$ and $q = i+ m(j-1)$, that
\begin{equation}\label{eq:column_of_bTheta}
    \btheta_{q,k} = \frac{p^{k-1}}N {\rm vec} \underbrace{(\one_m \one_m^\T\otimes\one_m \one_m^\T\otimes\cdots\otimes \bE_{ij} +\cdots+\bE_{ij}\otimes\cdots\otimes\one_m \one_m^\T\otimes\one_m \one_m^\T )}_{k~\text{times}},
\end{equation}
where we define the canonical matrix $\bE_{ij} \in \RR^{m \times m}$ in such a way that $[\bE_{ij}]_{uv} = \delta_{iu} \cdot \delta_{j v}$. 
Then, we have $\bTheta_k^\T \one_{m^{2k}} = \frac{p^{k-1}}N km^{2k-2} \one_{m^2}$ as a consequence of \Cref{lem:column_sum} (to be proven below), and therefore $\bTheta^\T \one_{N^2} = \frac{p^{K-1}K N}{m^2} \one_{m^2}$ by taking $k = K$.

In the proof above, we use the following two technical lemmas.


\begin{Lemma}\label{lem:column_sum}
For $1 \leq q \leq m^2, 1 \leq k \leq K$ and $\btheta_{q,k} \in \RR^{m^{2k}}$ defined as \eqref{eq:column_of_bTheta} with the understanding that the definition of $\btheta_{q,k}$ here disregards the coefficient $p^{k-1}/N$ for simplicity, one has
\begin{align*}
    \one_{m^{2k}}^\T \btheta_{q,k} = km^{2k-2}.
\end{align*}
\end{Lemma}
\begin{proof}[Proof of \Cref{lem:column_sum}]
For $1 \leq i,j \leq m$, we have $\sum_{u,v} [\bE_{ij}]_{uv} = 1$, so that for a fixed $k$, the value of $\one_{m^{2k}}^\T \btheta_{q,k}$ is \emph{independent} of $q$. 

Note that the entries of $\btheta_{q,k}$ is between $0$ and $k$. 
We use $\phi_{i,q,k}$ to denote the number of occurrences of element $i$ in $\btheta_{q,k}$. 
Then, for $2 \leq k \leq K$, 
\begin{align}
    \btheta_{q,k} ={\rm vec} \left(\one_{m^{k-1}}\one_{m^{k-1}}^\T \otimes \bE_q + {\rm mat}(\btheta_{q,k-1} )\otimes \one_m\one_m^\T \right),
\end{align}
where we use $\bE_q$ instead of $\bE_{ij}$ in \eqref{eq:column_of_bTheta}.

Thus, for $2 \leq k \leq K$, we have 
\begin{align*}
    \phi_{0,q,k}     &=(m^2-1) \phi_{0,q,k-1}, \\
    \phi_{i,q,k} &= (m^2-1) \phi_{i,q,k-1} + \phi_{i-1,q,k-1},\\
    \phi_{k,q,k} &= 1.
\end{align*}

By \Cref{lem:bTheta_column_recursive}, we have 
\begin{align*}
    \one_{m^{2k}}^\T \btheta_{q,k} &= \sum\limits_{i=1}^{k} i \phi_{i,q,k} \\
    &=  \sum\limits_{i=1}^{k} \frac{k!}{(i-1)!(k-i)!}(m^2-1)^{k-i} \\
    &= km^{2k-2},
\end{align*}
and thus the conclusion of the proof of \Cref{lem:column_sum}.
\end{proof}
\begin{Lemma}\label{lem:bTheta_column_recursive}
For $2 \leq k \leq K, 0 \leq i \leq k$, define $\{\phi_{i,k}\}$ as 
\begin{align*}
    \phi_{0,k}     &=(m^2-1) \phi_{0,k-1}, \\
    \phi_{i,k} &= (m^2-1) \phi_{i,k-1} + \phi_{i-1,k-1},\\
    \phi_{k,k} &= 1,
\end{align*}
where $\phi_{0,1} = m^2-1,\phi_{1,1} = 1$. Then a general formula of $\{\phi_{i,k}\}$ is  
\begin{align}\label{eq:bTheta_column_general_formula}
    \phi_{i,k} = \frac{k!}{i!(k-i)!}(m^2-1)^{k-i}.
\end{align}
\end{Lemma}

\begin{proof}[Proof of \Cref{lem:bTheta_column_recursive}]
For k = 2, we have that
\begin{align*}
    \phi_{0,2} = (m^2-1)^2, \quad \phi_{1,2} =2(m^2-1) , \quad\phi_{2,2} = 1.
\end{align*}

Then, assume that \eqref{eq:bTheta_column_general_formula} holds for $2\leq k \leq K-1$, then we will show that \eqref{eq:bTheta_column_general_formula} holds for $k+1$. For $1 \leq i \leq k+1$, we have that
\begin{align*}
    \phi_{0,k+1} &=(m^2-1) \phi_{0,k} = (m^2-1)^{k+1}, \\
    \phi_{i,k+1} &= (m^2-1) \phi_{i,k} + \phi_{i-1,k} \\
    &= \frac{k!}{i!(k-i)!}(m^2-1)^{k-i+1} +  \frac{k!}{(i-1)!(k-i+1)!}(m^2-1)^{k-i+1} \\
    &= \frac{(k+1)!}{i!(k-i+1)!}(m^2-1)^{k-i+1}.
\end{align*}
This thus concludes the proof of \Cref{lem:bTheta_column_recursive}.
\end{proof}

\subsubsection{Proof of \Cref{prop:linearized} Item~(iii)}
\label{subsec:SM_proof_Item_iii}

Item~(i)~and~(ii) of \Cref{prop:linearized} are already proven in the main text, it remains to prove Item~(iii) of \Cref{prop:linearized} by establishing, for $\bS_K$ as defined in \eqref{eq:def_S}, that
\begin{equation}
    \rank(\bS_K) \leq (m-1)K + 1.
\end{equation}
We will in fact show that for all $1\leq k \leq K$, one has 
\begin{equation}
  \rank(\bS_k) \leq (m-1)k + 1.
\end{equation}
To prove the above fact, 
\begin{itemize}
  \item[(i)] we first explore the iterative definition of $\bS_K$ in \eqref{eq:def_S} to write it as the sum of $K$ matrices $\M_{\ell,K}, \ell \in \{ 0, \ldots, K - 1\}$ of rank at most $m$ (which already provides an upper bound of the rank $\rank (\bS_K) \leq mK$); and
  \item[(ii)] with a more detailed analysis on how the (left and right) singular spaces of $\M_{\ell,K}$ ``intersect'' with each other when summing over $\ell$ to $\bS_K$, we can further tighten the upper bound to $\rank (\bS_K) \leq (m-1)K + 1$ as in the statement.
\end{itemize}

Recall the recursive definition of $\bS_k$ in \eqref{eq:def_S} as
\begin{align*}
    \bS_k = \frac{p^{k-1}}N ( \one_{m^{k-1}}\one_{m^{k-1}}^\T) \otimes \X + p \bS_{{k-1}} \otimes (\one_m\one_m^\T), \quad \bS_1 = \X/N,
\end{align*}
with $k = 1,2,\ldots,K$ and $N = m^K$.
It then follows from the iterative definition in \eqref{eq:def_S} that 
\begin{align*}
    \bS_k  &= p^{k-1}\one_{m^{k-1}}\one_{m^{k-1}}^\T \otimes \bS_1 + p\bS_{k-1}\otimes \one_{m}\one_{m}^\T  \\
        &=  p^{k-1}\one_{m^{k-1}}\one_{m^{k-1}}^\T \otimes \bS_1 + p (p^{k-2}\one_{m^{k-2}}\one_{m^{k-2}}^\T \otimes \bS_1 + p\bS_{k-2}\otimes \one_{m}\one_{m}^\T  )\otimes \one_{m}\one_{m}^\T \\
        & \ldots \\
        &= p^{k-1} \sum_{\ell=0}^{k-1} (\one_{m^\ell} \one_{m^\ell}^\T)\otimes\bS_1 \otimes (\one_{m^{k-\ell-1}} \one_{m^{k-\ell-1}}^\T).
\end{align*}

For $k \in \{1, \ldots, K \}$ and $0 \leq \ell \leq k-1$, denote the shortcut
\begin{equation}\label{eq:def_M}
    \M_{\ell,k} = (\one_{m^\ell} \one_{m^\ell}^\T)\otimes\bS_1 \otimes (\one_{m^{k-\ell-1}} \one_{m^{k-\ell-1}}^\T),
\end{equation}
one has 
\begin{equation}
    \bS_{k} = p^{k-1} \sum\limits_{\ell=0}^{k-1} \M_{\ell,k}. 
\end{equation}
Note that this already provides us with an upper bound of the rank, 
\begin{equation}
    \rank (\bS_K) \leq mK.
\end{equation}

To further improve this (upper bound) estimate of the rank of $\bS_K$, we need to perform a more detailed analysis of the singular spaces of $\M_{\ell,k}$, particularly when they are summed over $\ell$ to get $\bS_{k}$.

To that end, consider, without loss of generality that $\bS_1 = \X/N$ is of full rank (which indeed leads to an upper bound on the rank of $\bS_K$ eventually) the singular value decomposition (SVD) of $\bS_1$ as
\begin{equation}
    \bS_1 = \X/N = \U\bSigma\V^\T = \sum\limits_{i=1}^m \sigma_i \uu_i \vv_i^\T \in \RR^{m \times m},
\end{equation}  
with orthonormal $\U = [\uu_1, \ldots, \uu_m], \V = [\vv_1, \ldots, \vv_m] \in \RR^{m \times m}$ and diagonal $\bSigma \in \RR^{m \times m}$.


So that $\M_{\ell,k}$, as the Kronecker product between $\bS_1$ and matrices of all ones per its definition in \eqref{eq:def_M}, admits the following decomposition (which is almost an SVD but with ``unnormalized'' singular vectors),
\begin{align*}
    \M_{\ell,k} &= (\one_{m^\ell} \one_{m^\ell}^\T)\otimes\bS_1\otimes( \one_{m^{k-\ell-1}} \one_{m^{k-\ell-1}}^\T) \\
        &= \left( \one_{m^\ell} \one_{m^\ell}^\T \right) \otimes \left(\U\bSigma\V^\T\right)  \otimes \left(\one_{m^{k-\ell-1}} \one_{m^{k-\ell-1}}^\T  \right)  \\
        &=\sum_{i=1}^m \sigma_i(\one_{m^\ell} \otimes \uu_i \otimes \one_{m^{k-\ell-1}})(\one_{m^\ell} \otimes \vv_i \otimes \one_{m^{k-\ell-1}})^\T 
\end{align*}
where we used the fact that $(\A \otimes \B)(\C \otimes \mathbf{D}) = (\A \C)\otimes(\B \mathbf{D})$ with $\A, \B, \C, \mathbf{D}$ of appropriate dimension.


To further perform an in-depth analysis of how the left and right singular spaces of $\M_{\ell,K}$ intersect for different $\ell \leq K-1$, we introduce the following shortcuts 
\begin{align}
    \bar{\uu}_{\ell,i,k} &= \one_{m^\ell} \otimes \uu_i  \otimes \one_{m^{k-\ell-1}}, \label{eq:def_e} \\
    \bar{\vv}_{\ell,i,k} &=\one_{m^\ell} \otimes \vv_i \otimes \one_{m^{k-\ell-1}}, \label{eq:def_f}
\end{align}
so that $\M_{l,k}$ can be compactly rewritten as following sum of $m$ rank-one matrices,
\begin{align}
    \M_{\ell,k} =\sum_{i=1}^m \sigma_i \bar{\uu}_{\ell,i,k} \bar{\vv}_{\ell,i,k}^\T,
\end{align}
so that 
\begin{equation}\label{eq:decompose_S_k}
  \bS_{k} = p^{k-1} \sum_{\ell=0}^{k-1} \M_{\ell,k} = p^{k-1} \sum_{\ell=0}^{k-1} \sum_{i=1}^m \sigma_i \bar{\uu}_{\ell,i,k} \bar{\vv}_{\ell,i,k}^\T.
\end{equation}

In the following, we focus on the subspace spanned by the vectors of $\bar{\uu}_{\ell,i,k}$ (which in fact forms the left singular space of $\M_{\ell,k}$).
First note that by definition in \eqref{eq:def_e}, one has, for $1 \leq \ell \leq k-1$, the following recursive relation when increasing the value of $k$ or $\ell$,
\begin{equation}\label{eq:e_recursive}
  \bar{\uu}_{\ell,i,k+1} = \bar{\uu}_{\ell,i,k} \otimes \one_m, \quad \bar{\uu}_{\ell+1,i,k+1} = \one_m \otimes \bar{\uu}_{\ell,i,k}.
\end{equation}

A direct consequence of the recursion in \eqref{eq:e_recursive} is the following lemma, saying that for any $k \in \{ 1, \ldots, K\}$, the vector of all ones $\one_{m^k}$ is in the linear span of $\bar{\uu}_{0,1,k}, \bar{\uu}_{0,2,k}, \ldots,\bar{\uu}_{0,m,k}$.

\begin{Lemma}\label{lem:ones}
    For $1 \leq k \leq K$ and $\bar{\uu}_{\ell,i,k}$ defined as in \eqref{eq:def_e}, one has that
    \begin{equation}\label{eq:lem_ones_1}
      \one_{m^k} \in \spn \{ \bar{\uu}_{0,1,k}, \bar{\uu}_{0,2,k}, \ldots,\bar{\uu}_{0,m,k}\},
    \end{equation}
    and for $1 \leq k \leq K-1$ that
    \begin{equation}\label{eq:lem_ones_2}
      \one_{m^{k+1} } \in \spn \{ \bar{\uu}_{1,1,k+1}, \bar{\uu}_{1,2,k+1}, \ldots,\bar{\uu}_{1,m,k+1}\}.
    \end{equation}
\end{Lemma}
\begin{proof}[Proof of \Cref{lem:ones}]
We shall prove \Cref{lem:ones} based on an induction on the index $k$.
For $k=1$, we have that 
\begin{equation}
  \bar{\uu}_{0,1,1} = \uu_1, \bar{\uu}_{0,2,1} = \uu_2, \ldots,\bar{\uu}_{0,m,1} = \uu_m \in \RR^m,
\end{equation}
which, by definition, forms a basis of $\RR^m$, so that there exists a set of coefficients $\{\alpha_{i,1} \}_{i=1}^m$ such that
\begin{equation}\label{eq:ones_decompose}
    \one_{m} = \sum_{i=1}^m \alpha_{i,1} \uu_i = \sum_{i=1}^m \alpha_{i,1} \bar{\uu}_{0,i,1}.
\end{equation}

Then, assume that \eqref{eq:lem_ones_1} holds for $1 \leq k \leq K -1$, so that there exists a set of coefficients $\{\alpha_{i,k}\}_{i=1}^m$ such that
\begin{equation}
  \one_{m^k} = \sum_{i=1}^m \alpha_{i,k} \bar{\uu}_{0,i,k}.
\end{equation}
Then, one has
\begin{align*}
    \one_{m^{k+1} } = \one_{m^k } \otimes \one_{m} &= \left( \sum_{i=1}^m \alpha_{i,k} \bar{\uu}_{0,i,k} \right) \otimes \one_{m} = \sum_{i=1}^m \alpha_{i,k} \bar{\uu}_{0,i,k+1},
\end{align*}
where we used the iterative relation in \eqref{eq:e_recursive}.
This allows us to conclude the proof of \eqref{eq:lem_ones_1} in \Cref{lem:ones}.

For \eqref{eq:lem_ones_2}, it suffices to write, with \eqref{eq:lem_ones_1} that
\begin{equation}
  \one_{m^k} = \sum_{i=1}^m \alpha_{i,k} \bar{\uu}_{0,i,k},
\end{equation}
for some coefficients $\{\alpha_{i,k}\}_{i=1}^m$, so that
\begin{equation}
  \one_{m^{k+1}} = \one_m \otimes \one_{m^k} = \one_m \otimes \left(\sum_{i=1}^m \alpha_{i,k} \bar{\uu}_{0,i,k} \right) = \sum_{i=1}^m \alpha_{i,k} \bar{\uu}_{1,i,k+1},
\end{equation}

This thus concludes the proof of \Cref{lem:ones}.
\end{proof}


With the recursion in \eqref{eq:e_recursive} and \Cref{lem:ones} at hand, we are now already to characterize the precise ``interaction'' of the left singular space of $\M_{\ell,k}$ and that of $\M_{\ell',k}$ with $\ell' \leq \ell$.
This is described in the following result.

\begin{Lemma}\label{lem:linear_space}
    For $2 \leq k \leq K$, define the linear spans of vectors $\mathcal{A}_{k}, \mathcal{B}_{k} \subseteq \RR^{m^k}$ as,
    \begin{align*}
    \mathcal{A}_{k} &= \spn \{ \bar{\uu}_{0,1,k}, \ldots, \bar{\uu}_{0,m-1,k}, \bar{\uu}_{1,1,k}, \ldots, \bar{\uu}_{1,m-1,k}, \ldots, \bar{\uu}_{k-1,1,k}, \ldots, \bar{\uu}_{k-1,m-1,k}, \bar{\uu}_{k-1,m,k}\}, \\
    \mathcal{B}_{k} &= \spn \{ \bar{\uu}_{0,m,k}, \ldots, \bar{\uu}_{k-2,m,k} \}.
    \end{align*}
    Then, one has $\mathcal{B}_{k} \subseteq \mathcal{A}_{k}$.
\end{Lemma}

\begin{proof}[Proof of \Cref{lem:linear_space}]
We will prove \Cref{lem:linear_space} again using a mathematical induction on the index $k$.

First, in the case $k = 2$, it suffices to show that
\begin{equation}
  \mathcal{B}_2 = \bar{\uu}_{0,m,2} \in \mathcal{A}_2 = \spn \{ \bar{\uu}_{0,1,2}, \ldots, \bar{\uu}_{0,m-1,2}, \bar{\uu}_{1,1,2}, \ldots, \bar{\uu}_{1,m-1,2}, \bar{\uu}_{1,m,2}\}.
\end{equation}
This allows in a straightforward manner from the fact that,
\begin{equation}
  \one_{m^2 } \in \spn \{ \bar{\uu}_{1,1,2}, \bar{\uu}_{1,2,2}, \ldots,\bar{\uu}_{1,m,2}\}.
\end{equation}
by taking $k = 1$ in \eqref{eq:lem_ones_2}, as well as 
\begin{equation}
  \one_{m^2} \in \spn \{ \bar{\uu}_{0,1,2}, \bar{\uu}_{0,2,2}, \ldots,\bar{\uu}_{0,m,2}\},
\end{equation}
by taking $k = 2$ in \eqref{eq:lem_ones_1}, so that $\bar{\uu}_{0,m,2} \in \mathcal{A}_2$.

Now, assume that $\mathcal{B}_{k} \subseteq \mathcal{A}_{k}$ holds, we would like to show that $\mathcal{B}_{k+1} \subseteq \mathcal{A}_{k+1}$. 

Let
\begin{equation}
  \mathcal{C}_k = \spn  \{  \bar{\uu}_{1,1,k}, \ldots, \bar{\uu}_{1,m-1,k}, \ldots, \bar{\uu}_{k-1,1,k}, \ldots, \bar{\uu}_{k-1,m-1,k},\bar{\uu}_{k-1,m,k}\} \subseteq \mathcal{A}_k,
\end{equation}
we have, by the recursive relation in \eqref{eq:e_recursive} and $\mathcal{B}_{k} \subseteq \mathcal{A}_{k}$, that
\begin{align}
  &\bar{\uu}_{1,m,k+1}, \bar{\uu}_{2,m,k+1}, \ldots, \bar{\uu}_{k-1,m,k+1} \nonumber \\ 
  &= \one_m \otimes \bar{\uu}_{0,m,k}, \one_m \otimes \bar{\uu}_{1,m,k}, \ldots, \one_m \otimes  \bar{\uu}_{k-2,m,k} \in \mathcal{C}_{k+1} \subseteq \mathcal{A}_{k+1}. \label{eq:in_C}
\end{align}
It thus remains to show that
\begin{equation}
  \bar{\uu}_{0,m,k+1} \in \mathcal{A}_{k+1},
\end{equation}
to reach the conclusion of $\mathcal{B}_{k+1} \subseteq \mathcal{A}_{k+1}$.

To this end, by \eqref{eq:lem_ones_1} in Lemma~\ref{lem:ones}, we have that
\begin{equation}
    \bar{\uu}_{0,m,k+1} \in \spn \{\bar{\uu}_{0,1,k+1},\bar{\uu}_{0,2,k+1},\ldots,\bar{\uu}_{0,m-1,k+1},\one_{m^{k+1}} \},
\end{equation}
since
\begin{equation}
  \one_{m^{k+1}} =\one_m \otimes \one_{m^k}, \quad \one_{m^k} \in \spn \{ \bar{\uu}_{0,1,k},\bar{\uu}_{0,2,k},\ldots,\bar{\uu}_{0,m,k}\},
\end{equation}
we thus get, using again \eqref{eq:e_recursive}, that
\begin{equation}
  \one_{m^{k+1}} \in \spn \{ \bar{\uu}_{1,1,k+1},\bar{\uu}_{1,2,k+1},\ldots,\bar{\uu}_{1,m,k+1}\}.
\end{equation}

Now, recall from \eqref{eq:in_C} that $\bar{\uu}_{1,m,k+1} \in \mathcal{C}_{k+1}$, so that we have
\begin{equation}
  \one_{m^{k+1}} \in \mathcal{C}_{k+1},
\end{equation}
and therefore
\begin{equation}
  \bar{\uu}_{0,m,k+1} \in \spn \{\bar{\uu}_{0,1,k+1},\bar{\uu}_{0,2,k+1},\ldots,\bar{\uu}_{0,m-1,k+1} \} \mathop{\cup} \mathcal{C}_{k+1} \subseteq \mathcal{A}_{k+1}.
\end{equation}
This allows us to conclude that $\mathcal{B}_{k+1} \subseteq \mathcal{A}_{k+1}$ and thus the conclusion of the proof of Lemma~\ref{lem:linear_space}. 

\end{proof}

In plain words, \Cref{lem:linear_space} tells that, for given $k \leq K$, the left singular space of $\M_{\ell,k}$, when summing over $\ell$ to form the matrix of interest $\bS_k = p^{k-1} \sum_{\ell=0}^{k-1} \M_{\ell,k}$, intersects with each other in such a way that $\rank(\bS_k) < m k$.

To obtain a tight upper bound of $\rank(\bS_k)$ (than $m k$), it follows from \Cref{lem:linear_space} that for $2\leq k\leq K$ and $0 \leq j \leq k-2$, the basis vectors $\bar{\uu}_{j,m,k}$ of $\mathcal B_{k}$ can be written as the following linear combination
\begin{equation}
  \bar{\uu}_{j,m,k} = \sum_{\ell=0}^{k-1}\sum_{i=1}^{m-1} \gamma_{j,\ell,i,k}\bar{\uu}_{\ell,i,k} +\gamma_{j,k-1,m,k}\bar{\uu}_{k-1,m,k},
\end{equation}
for some set of coefficients $\{\gamma_{j,\ell,i,k}\}$.
Then, it follows from \eqref{eq:decompose_S_k} that
\begin{align*}
         p^{1-k}\bS_k &=\sum\limits_{\ell=0}^{k-1}\sum\limits_{i=1}^m \sigma_i \bar{\uu}_{\ell,i,k} \bar{\vv}_{\ell,i,k}^\T \\
         &= \sum\limits_{\ell=0}^{k-1}\sum\limits_{i=1}^{m-1} \sigma_i \bar{\uu}_{\ell,i,k} \bar{\vv}_{\ell,i,k}^\T 
         +\sum\limits_{\ell=0}^{k-2} \sigma_m \bar{\uu}_{\ell,m,k} \bar{\vv}_{\ell,m,k}^\T + \sigma_m \bar{\uu}_{k-1,m,k} \bar{\vv}_{k-1,m,k}^\T \\
         &= \sum\limits_{\ell=0}^{k-1}\sum\limits_{i=1}^{m-1} \sigma_i \bar{\uu}_{\ell,i,k} \bar{\vv}_{\ell,i,k}^\T 
         +\sum\limits_{\ell=0}^{k-2} \sigma_m \left( \sum_{\ell'=0}^{k-1}\sum_{i=1}^{m-1} \gamma_{\ell,\ell',i,k} \bar{\uu}_{\ell',i,k} + \gamma_{\ell,k-1,m,k}\bar{\uu}_{k-1,m,k} \right) \bar{\vv}_{\ell,m,k}^\T  \\
         &+ \sigma_m \bar{\uu}_{k-1,m,k} \bar{\vv}_{k-1,m,k}^\T \\
         &= \sum\limits_{\ell=0}^{k-1}\sum\limits_{i=1}^{m-1} \bar{\uu}_{\ell,i,k}\left(\sigma_i\bar{\vv}_{\ell,i,k}^\T + \sum\limits_{\ell'=0}^{k-2} \sigma_m \gamma_{\ell',\ell,i,k}\bar{\vv}_{\ell',m,k}^\T \right) \\ 
         &+ \sigma_m\bar{\uu}_{k-1,m,k}\left(\bar{\vv}_{k-1,m,k}^\T + \sum\limits_{\ell'=0}^{k-2} \gamma_{\ell',k-1,m,k}\bar{\vv}_{\ell',m,k}^\T \right),
\end{align*}
where in the last equality we exchanged the index $\ell$ and $\ell'$ for the ease of exposition, so that $\bS_k$, as the sum of $(m-1)k + 1$ matrices of rank-one, satisfies 
\begin{equation}
  \rank(\bS_k) \leq (m-1)k + 1.
\end{equation}
Also, note that in passing we have shown that $\one_{m^{k}} \in \mathcal{A}_{k}$, so that we have similarly that 
\begin{equation}
    \rank(\bP_K^{\rm lin}) = \rank( p^K \one_{N} \one_{N}^\T + \sqrt N \bS_K ) \leq (m-1)K + 1.
\end{equation}
with $N = m^K$.
This thus allows us to conclude of the proof of Item~(iii) in Proposition~\ref{prop:linearized}.

\subsection{Proof of \Cref{theo:S+N_for_A}}
\label{subsec:proof_info_plus_noise_A}

By \Cref{def:random_Kronecker_graph}, we have $\A = \bPi (\bP_K + \Z) \bPi^{-1} = \bPi \bP_K^{\rm lin} \bPi^{-1} + \Z + \tilde O_{\| \cdot \|_2} (1) $, where we used the fact that $ \| \bP_K - \bP_K^{\rm lin} \|_2 = \tilde O(1)$ from \Cref{prop:linearized} and that the distribution of $\Z$ is invariant after permuted by $\bPi$.
This concludes the proof of \Cref{theo:S+N_for_A}.

\subsection{Proof of \Cref{lem:estimate_p}}
\label{subsec:proof_lem_estimate_p}

By \Cref{theo:S+N_for_A}, we have $\A = \bPi \bP_K \bPi^{-1} + \Z $.
First, we can note that $\one_N^\T \bPi \bP_K \bPi^{-1} \one_N = \one_N^\T \bP_K \one_N = p^K N^2 + \sqrt N \one_N^\T \bS_K \one_N + \tilde O(N) = p^K N^2 + \tilde O(N^{3/2})$,
where we used $\| \bS_K \|_{\max} = \tilde O(N^{-1})$ so that $\sqrt N \one_N^\T \bS_K \one_N = \tilde O(N^{3/2})$.
It then follows from the strong law of large numbers that $\frac1{N^2} \one_N^\T \Z \one_N \to 0$ almost surely as $N \to \infty$, and thus the conclusion.

\subsection{Proof of \Cref{prop:approx_centered_adj}}
\label{subsec:proof_prop_approx_centered_adj}

Note from the proof of \Cref{lem:estimate_p} that $\one_N^\T \A \one_N = \one_N^\T \bP_K \one_N + O(N) = \one_N^\T \bP_K^{\rm lin} \one_N + \tilde O(N) = p^K N^2 + \sqrt N \one_N^\T \bS_K \one_N + \tilde O(N)$.
Moreover, by Item~(ii) of \Cref{prop:linearized} and the assumption that $\one_m^\T \X \one_m = \tilde O(N^{-1/2})$, we have $\one_N^\T \bS_K \one_N = \one_{N^2}^\T \bTheta {\rm vec}(\X) = \frac{p^{K-1} KN}{m^2} \one_{m^2}^\T{\rm vec}(\X)  = \tilde O(\sqrt N)$, so that $\frac1{N^2} \one_N^\T \A \one_N = p^K + \tilde O(N^{-1})$ and $\frac{ \one_N^\T \A \one_N }{N^2} \one_N \one_N^\T = p^K \one_N \one_N^\T + \tilde O_{\| \cdot \|_2}(1)$.
This concludes the proof of \Cref{prop:approx_centered_adj}.


\section{Spectral Analysis of Random Kronecker Graphs}
\label{sec:spectral_analysis_centered_adjacency}

In this section, we provide some additional theoretical and empirical results on the spectra of large random Kronecker graphs.
With \Cref{theo:S+N_for_A}~and~\Cref{prop:approx_centered_adj} at hand, we have the following result on the asymptotic singular spectral characterization of the centered adjacency $\bar \A = \frac1{\sqrt N} \left(  \A - \frac{ \one_N^\T \A \one_N }{N^2} \one_N \one_N^\T \right)$ defined in \eqref{eq:def_centered_adj}.

\begin{Theorem}[Asymptotic characterization of adjacency spectrum]\label{theo:bar_A}
Under the notations and setting of \Cref{prop:approx_centered_adj}, the empirical singular value distribution $\mu_{\bar \A}$, defined as the normalized countering measure $\mu_{\bar \A} \equiv \frac1N \sum_{i=1}^N \delta_{\hat \sigma_i}$ of $\hat \sigma_i$, the singular values (listed in a decreasing order) of the centered adjacency $\bar \A $ in \eqref{eq:def_centered_adj} with $p^K \to \bar p \in (0,1)$, converges weakly to
\begin{equation}\label{eq:QC_law}
    \mu(dx) = \frac{\sqrt{ 4 \bar p (1-\bar p) - {x^2 }} }{\bar p(1-\bar p) \pi}  \cdot 1_{ \left[0,2\sqrt{\bar p(1-\bar p)} \right]}(x) \,dx,
\end{equation}
with probability approaching one as $N \to \infty$, known as the (rescaled) quarter-circle law~\cite{bai2010spectral}.
Moreover, let $ \ell_i = \lim_{N \to \infty} \sigma_i(\bS_K^{\bPi})/\sqrt{ \bar p(1- \bar p) }$ and $\ell_1 > \ell_2 > \cdots > \ell_{\rank(\bS_K^{\bPi})}$ with $\sigma_i(\bS_K^{\bPi})$ the $i$th largest singular value of $\bS_K^{\bPi}$ defined in \eqref{eq:def_S_K^Pi}, with associated left and right singular vectors $\uu_i$ and $\vv_i$, then, the top singular values as well as the associated (left and right) singular vector triples $(\hat\sigma_i, \hat \uu_i, \hat \vv_i)$ of $\bar \A$ establish the following \emph{phase transition} behavior
\begin{equation}\label{eq:PT_sigma_A}
    \hat \sigma_i \to \begin{cases} \sqrt{ \bar p (1-\bar p) (2 + \ell_i^2 + \ell_i^{-2}) }, & \ell_i >1, \\ 2\sqrt{\bar p (1-\bar p)}, & \ell_i \leq 1; \end{cases}
\end{equation}
and for $1 \leq i \leq \rank(\bS_K^{\bPi})$, $1 \leq j \leq N$,
\begin{equation}
    (\uu_i^\T \hat \uu_j)^2 \to (1 - \ell_i^{-2}) \cdot 1_{\ell_i \geq 1} \cdot 1_{i=j}, \quad  (\vv_i^\T \hat \vv_j)^2 \to (1 - \ell_i^{-2}) \cdot 1_{\ell_i \geq 1} \cdot 1_{i=j}.
\end{equation}
\end{Theorem}
\begin{proof}[Proof of \Cref{theo:bar_A}]
The singular values $\sigma_i(\bar \A)$ of $\bar \A \in \RR^{N \times N}$ are the square root of the corresponding eigenvalues $\lambda_i(\bar \A \bar \A^\T)$ of $ \bar \A \bar \A^\T$, i.e., $\sigma_i(\bar \A) = \sqrt{ \lambda_i(\bar \A \bar \A^\T) }$, and it thus suffices to evaluate the eigenvalues and the corresponding eigenvectors of the positive semi-definite matrix $ \bar \A \bar \A^\T$ and $ \bar \A^\T \bar \A$.
It then follows from \Cref{prop:approx_centered_adj} that $\bar \A$ can be decomposed, for $N$ large, as the sum of a zero-mean random matrix $\Z/\sqrt N$ and a small-rank deterministic signal matrix $\bS_K^{\bPi}$ as $ \bar \A  = \Z/\sqrt N + \bS_K^{\bPi} + \tilde O_{\| \cdot \|_2}(N^{-1/2})$.
The asymptotic characterization of eigenvalues and eigenvectors of sample covariance matrices $\bar \A \bar \A^\T$ or $\bar \A^\T \bar \A$ are rather standard in the random matrix literature, but \emph{only} when the the rank of the signal matrix $\bS_K^{\bPi}$ is \emph{fixed} with respect to its dimension $N \to \infty$, see for example \cite{johnstone2001distribution,peche2006largesta,baik2006eigenvalues,benaych2012singular,bai2012sample,hachem2012large,gavish2017optimal,donoho2018optimal,couillet2022RMT4ML} and the references therein.

Here, we are in the setting where the rank of $\bS_K^{\bPi}$ grows with the dimension $N$, but very slowly in the sense that $ \rank( \bS_K^{\bPi} )\leq (m-1)K + 1 = (m-1) \log_m(N) + 1 = o(N)$, as shown in Item~({\romannumeral 3}) of \Cref{prop:linearized}.
And it suffices to apply the deterministic equivalent result, e.g., \cite[Theorem~2.4]{couillet2022RMT4ML}, and note that the resulting approximation errors are of the order $O(\log(N)N^{-1/4}) $ for $\Z$ having bounded and thus sub-gaussian entries.
This concludes the proof of \Cref{theo:bar_A}.
\end{proof}

Let $\bS_K^{\bPi} = \sum_{i=1}^{\rank(\bS_K)} \sigma_i(\bS_K^{\bPi}) \uu_i \vv_i^\T$ denote the singular value decomposition (SVD) of $\bS_K^{\bPi}$, a first estimate of $\bS_K^{\bPi}$ is to apply the \emph{hard thresholding} (HS) on the SVD of the noisy centered adjacency $\bar \A$ as in \eqref{eq:def_hard_thresholding}
\begin{equation}
    \hat \bS_{K}^{\rm HS} = \textstyle \sum_{i=1}^{\rank(\bS_K)} \hat \sigma_i \hat \uu_i \hat \vv_i^\T,
\end{equation}
with $(\hat\sigma_i, \hat \uu_i, \hat \vv_i)$ the singular values (listed in a decreasing order) and singular vector triples of $\bar \A$.
We know, however from \Cref{theo:bar_A} that this first estimate, despite taking a simple form and minimizes the spectral norm difference $\| \hat \bS_{K}^{\rm HS} - \bar \A \|_2$ under the constraint of having rank $\rank(\bS_K)$ (as a consequence of the Eckart--Young--Mirsky theorem, see \cite{eckart1936approximation,mirsky1960symmetric}), is a ``biased'' estimate of the object of interest $\bS_K^{\bPi}$ for $N$ large, in the following sense:
\begin{itemize}
    \item[(i)] when the signal-to-noise-ratio (SNR) $\ell_i = \sigma_i(\bS_K^{\bPi})/\sqrt{\bar p (1- \bar p)}$ of $\bS_K^{\bPi}$ defined in \Cref{theo:bar_A} is \emph{below} the phase transition threshold $1$, the corresponding $\hat\sigma_i$ is \emph{independent} of $\sigma_i(\bS_K^{\bPi})$, with singular vectors asymptotically orthogonal to the true $\uu_i$ and $\vv_i$; and
    \item[(ii)] even for SNR \emph{above} the threshold, one still has $\hat \sigma_i \neq \sigma_i( \bS_K^{\bPi} )$ and that there is a non-trivial ``angle'' between $\hat \uu_i$ and $\uu_i$ (and similarly between $\hat \vv_i$ and $\vv_i$), unless the SNR $\ell_i \to \infty$. 
\end{itemize}

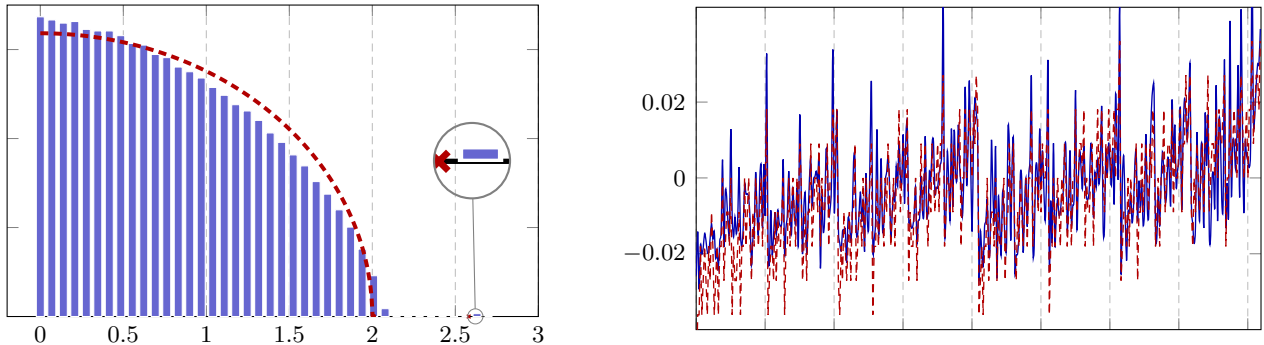
\begin{figure}[thb]
  \centering
  \begin{minipage}[c]{0.42\textwidth}
  \centering
  \begin{tikzpicture}[font=\footnotesize,spy using outlines]
    \renewcommand{\axisdefaulttryminticks}{4} 
    \pgfplotsset{every major grid/.append style={densely dashed}}          
    \tikzstyle{every axis y label}+=[yshift=-10pt] 
    \tikzstyle{every axis x label}+=[yshift=5pt]
    \pgfplotsset{every axis legend/.style={cells={anchor=west},fill=white,at={(0.98,0.98)}, anchor=north east, font=\footnotesize}}
    \begin{axis}[
      width=1.2\linewidth,
      height=.8\linewidth,
      xmin=-.2,xmax=3,
      ymin=0,ymax=0.7,
        yticklabels = {},
        bar width=3pt,
        grid=major,
        ymajorgrids=false,
        scaled ticks=false,
        xlabel={},
        ylabel={}
        ]
        \addplot+[ybar,mark=none,color=white,fill=BLUE!60!white,area legend] coordinates{
        (0.000000,0.673557)(0.069231,0.666504)(0.138462,0.659451)(0.207692,0.662977)(0.276923,0.645345)(0.346154,0.641819)(0.415385,0.641819)(0.484615,0.631239)(0.553846,0.613607)(0.623077,0.610080)(0.692308,0.588921)(0.761538,0.581868)(0.830769,0.560710)(0.900000,0.550130)(0.969231,0.536024)(1.038462,0.514865)(1.107692,0.497233)(1.176923,0.476074)(1.246154,0.461968)(1.315385,0.440809)(1.384615,0.412598)(1.453846,0.391439)(1.523077,0.363227)(1.592308,0.338542)(1.661538,0.303277)(1.730769,0.275065)(1.800000,0.239800)(1.869231,0.201009)(1.938462,0.151638)(2.007692,0.091688)(2.076923,0.017632)(2.146154,0.000000)(2.215385,0.000000)(2.284615,0.000000)(2.353846,0.000000)(2.423077,0.000000)(2.492308,0.000000)(2.561538,0.000000)(2.630769,0.006526)(2.700000,0.000000)
        };
        \addplot[densely dashed,RED,samples=300,domain=0:2.05,line width=1.5pt] {sqrt(max( 4-x^2 ,0))/pi};
        \addplot+[only marks,mark=x,RED,line width=.5pt,mark size=1pt] plot coordinates{(15.9836, 0)};
        \coordinate (spypoint1) at (axis cs:2.62,0);
        \coordinate (magnifyglass1) at (axis cs:2.6,0.35);
        \addplot+[only marks,mark=x,RED,line width=.5pt,mark size=1pt] plot coordinates{(2.5806, 0)};
        \end{axis}
        \spy[black!50!white,size=1cm,circle,connect spies,magnification=5] on (spypoint1) in node [fill=none] at (magnifyglass1);
  \end{tikzpicture}
  \end{minipage}
  \hfill{}%
  \begin{minipage}[c]{0.53\textwidth}
  \centering
  \begin{tikzpicture}[font=\footnotesize]
    \renewcommand{\axisdefaulttryminticks}{4} 
    \pgfplotsset{every major grid/.append style={densely dashed}}          
    \tikzstyle{every axis y label}+=[yshift=-10pt] 
    \tikzstyle{every axis x label}+=[yshift=5pt]
    \pgfplotsset{every axis legend/.style={cells={anchor=west},fill=white,at={(0,1)}, anchor=north west, font=\footnotesize}}
    \begin{axis}[
      width=1\linewidth,
      height=.65\linewidth,
      xmin=0,xmax=4096,
      ymin=-0.04,ymax=0.045,
        ytick = {-0.02, 0, 0.02},
        yticklabels = {$-0.02$, $0$, $0.02$},
        xticklabels = {},
        grid=major,
        ymajorgrids=false,
        scaled ticks=false,
        xlabel={ },
        ylabel={ }
        ]
        \addplot[smooth,BLUE,line width=.5pt] plot coordinates{
        (1,-0.025723)(11,-0.014164)(21,-0.029371)(31,-0.017719)(41,-0.020135)(51,-0.018371)(61,-0.014631)(71,-0.016498)(81,-0.019613)(91,-0.019569)(101,-0.015997)(111,-0.012018)(121,-0.016170)(131,-0.014025)(141,-0.021186)(151,-0.018814)(161,-0.023728)(171,-0.016589)(181,-0.013967)(191,0.004813)(201,-0.018200)(211,-0.003664)(221,-0.004915)(231,-0.012415)(241,-0.013808)(251,0.012898)(261,-0.010016)(271,-0.015598)(281,-0.010347)(291,-0.014777)(301,-0.010771)(311,0.002830)(321,-0.014657)(331,-0.017995)(341,-0.013121)(351,-0.001712)(361,-0.004237)(371,-0.004153)(381,-0.005304)(391,-0.011189)(401,-0.023057)(411,-0.016185)(421,-0.013698)(431,0.001558)(441,-0.008753)(451,-0.014994)(461,-0.000352)(471,0.003196)(481,-0.013282)(491,0.003634)(501,-0.009409)(511,0.032806)(521,-0.017398)(531,-0.022137)(541,-0.008524)(551,-0.020217)(561,-0.015294)(571,-0.009282)(581,-0.014886)(591,-0.008463)(601,-0.011610)(611,-0.012731)(621,-0.011126)(631,0.004889)(641,-0.020120)(651,-0.013256)(661,-0.009675)(671,0.001789)(681,-0.008306)(691,-0.014878)(701,-0.002893)(711,-0.009332)(721,-0.012119)(731,-0.002525)(741,-0.017877)(751,0.016730)(761,-0.002274)(771,-0.020090)(781,-0.007478)(791,-0.004236)(801,-0.019890)(811,-0.011062)(821,-0.000579)(831,0.005068)(841,-0.007306)(851,-0.000486)(861,0.000318)(871,0.001783)(881,-0.001052)(891,0.005071)(901,-0.023839)(911,-0.004921)(921,-0.014132)(931,-0.011073)(941,0.006916)(951,0.004536)(961,-0.009613)(971,-0.008250)(981,-0.001343)(991,0.033863)(1001,-0.008365)(1011,0.002410)(1021,0.017441)(1031,-0.012843)(1041,-0.017133)(1051,-0.019233)(1061,-0.014089)(1071,-0.013845)(1081,-0.013654)(1091,-0.022298)(1101,-0.010920)(1111,-0.011620)(1121,-0.016507)(1131,-0.011594)(1141,0.004208)(1151,0.005287)(1161,-0.013764)(1171,-0.014978)(1181,-0.009518)(1191,-0.003634)(1201,-0.011512)(1211,-0.006036)(1221,-0.019451)(1231,0.001223)(1241,-0.011137)(1251,-0.013091)(1261,0.007248)(1271,0.025081)(1281,-0.018091)(1291,-0.016053)(1301,-0.011057)(1311,0.012775)(1321,-0.013516)(1331,-0.011006)(1341,0.005120)(1351,-0.013452)(1361,-0.011770)(1371,-0.000972)(1381,0.001504)(1391,0.005429)(1401,-0.000703)(1411,-0.016078)(1421,-0.010037)(1431,-0.003810)(1441,-0.006176)(1451,0.003156)(1461,-0.002235)(1471,0.009684)(1481,-0.000058)(1491,-0.001189)(1501,0.011465)(1511,-0.006757)(1521,0.011160)(1531,0.013675)(1541,-0.010289)(1551,-0.011162)(1561,-0.013204)(1571,-0.012659)(1581,-0.009455)(1591,-0.002416)(1601,-0.017242)(1611,-0.011483)(1621,-0.010220)(1631,0.000540)(1641,-0.008708)(1651,-0.009392)(1661,0.010995)(1671,-0.009212)(1681,-0.008806)(1691,0.001679)(1701,-0.000796)(1711,0.009152)(1721,0.009476)(1731,-0.004070)(1741,-0.002554)(1751,-0.002578)(1761,0.001221)(1771,0.003572)(1781,0.008485)(1791,0.048542)(1801,-0.012313)(1811,-0.009980)(1821,0.006933)(1831,-0.000652)(1841,-0.011009)(1851,0.012999)(1861,-0.007643)(1871,0.012695)(1881,0.001297)(1891,-0.000322)(1901,0.014009)(1911,0.002419)(1921,-0.002001)(1931,-0.010049)(1941,0.000731)(1951,0.024020)(1961,-0.001634)(1971,0.007361)(1981,0.025664)(1991,0.002984)(2001,-0.014127)(2011,0.009752)(2021,0.021023)(2031,0.017007)(2041,0.011152)(2051,-0.026555)(2061,-0.014085)(2071,-0.019574)(2081,-0.023310)(2091,-0.012988)(2101,-0.011272)(2111,-0.005703)(2121,-0.018570)(2131,-0.018244)(2141,0.009757)(2151,-0.002816)(2161,-0.011395)(2171,0.008425)(2181,-0.026483)(2191,-0.005127)(2201,-0.011377)(2211,-0.020426)(2221,-0.004215)(2231,-0.008097)(2241,-0.017915)(2251,-0.017218)(2261,0.012687)(2271,0.008450)(2281,-0.016496)(2291,0.003917)(2301,0.009692)(2311,-0.024200)(2321,-0.020431)(2331,-0.019912)(2341,-0.014405)(2351,0.000993)(2361,-0.000684)(2371,-0.019678)(2381,-0.007937)(2391,-0.000007)(2401,-0.012864)(2411,0.000397)(2421,0.005491)(2431,0.026587)(2441,-0.017191)(2451,-0.018448)(2461,0.001450)(2471,0.003365)(2481,-0.005635)(2491,0.012988)(2501,-0.001965)(2511,0.007196)(2521,-0.002834)(2531,-0.011245)(2541,0.001000)(2551,0.031123)(2561,-0.023920)(2571,-0.005994)(2581,-0.016901)(2591,0.009133)(2601,-0.017756)(2611,-0.009975)(2621,-0.005636)(2631,-0.009749)(2641,-0.009073)(2651,-0.004607)(2661,-0.004400)(2671,0.006968)(2681,0.008071)(2691,-0.013529)(2701,-0.005039)(2711,0.007231)(2721,-0.012135)(2731,-0.007408)(2741,-0.004479)(2751,0.023149)(2761,-0.013246)(2771,-0.000933)(2781,0.006976)(2791,0.002657)(2801,-0.002844)(2811,0.008521)(2821,-0.000754)(2831,-0.005188)(2841,-0.005669)(2851,-0.002774)(2861,0.008859)(2871,0.004634)(2881,-0.010422)(2891,0.001429)(2901,-0.002546)(2911,0.004448)(2921,-0.001655)(2931,0.003256)(2941,0.018689)(2951,-0.012918)(2961,-0.013041)(2971,0.007255)(2981,0.000600)(2991,0.009291)(3001,0.003704)(3011,-0.001964)(3021,0.007924)(3031,-0.006990)(3041,0.021872)(3051,0.011395)(3061,0.000632)(3071,0.046262)(3081,-0.024526)(3091,-0.000266)(3101,0.005633)(3111,-0.009199)(3121,-0.021777)(3131,-0.009213)(3141,-0.009154)(3151,-0.004050)(3161,-0.011765)(3171,-0.007930)(3181,0.003519)(3191,0.006999)(3201,-0.011444)(3211,-0.000731)(3221,-0.007917)(3231,-0.001411)(3241,-0.012327)(3251,-0.008572)(3261,0.011537)(3271,-0.007432)(3281,-0.010439)(3291,0.005115)(3301,0.017749)(3311,0.023629)(3321,0.022021)(3331,-0.013994)(3341,-0.015598)(3351,-0.002662)(3361,-0.012785)(3371,-0.001529)(3381,0.006892)(3391,0.001900)(3401,-0.006239)(3411,-0.004409)(3421,0.005356)(3431,0.001425)(3441,-0.007016)(3451,0.020674)(3461,0.004093)(3471,0.014035)(3481,-0.002439)(3491,0.002095)(3501,0.004988)(3511,0.018783)(3521,-0.008459)(3531,0.012001)(3541,0.020826)(3551,0.016881)(3561,0.010135)(3571,0.016992)(3581,0.029818)(3591,-0.006621)(3601,-0.012100)(3611,-0.017275)(3621,-0.012664)(3631,0.012084)(3641,-0.003051)(3651,-0.011238)(3661,0.007753)(3671,0.006763)(3681,-0.013205)(3691,0.016958)(3701,0.003268)(3711,0.016364)(3721,-0.015082)(3731,-0.015196)(3741,-0.000369)(3751,0.000998)(3761,-0.007608)(3771,0.024450)(3781,-0.011378)(3791,0.014560)(3801,0.012309)(3811,-0.005606)(3821,-0.007024)(3831,0.031224)(3841,-0.012382)(3851,-0.001366)(3861,-0.009494)(3871,0.041406)(3881,-0.004234)(3891,0.009189)(3901,0.009332)(3911,0.002500)(3921,0.011302)(3931,0.029765)(3941,-0.005142)(3951,0.044454)(3961,0.014779)(3971,-0.009279)(3981,0.009118)(3991,0.005926)(4001,0.004543)(4011,0.002834)(4021,0.009353)(4031,0.058156)(4041,0.015480)(4051,0.016703)(4061,0.024396)(4071,0.029868)(4081,0.030112)(4091,0.039283)
        };
        \addplot[densely dashed,RED,line width=0.5pt] plot coordinates{
        (1,-0.054127)(11,-0.036084)(21,-0.036084)(31,-0.018042)(41,-0.036084)(51,-0.027063)(61,-0.018042)(71,-0.027063)(81,-0.036084)(91,-0.018042)(101,-0.027063)(111,-0.009021)(121,-0.018042)(131,-0.036084)(141,-0.027063)(151,-0.018042)(161,-0.036084)(171,-0.018042)(181,-0.018042)(191,-0.000000)(201,-0.027063)(211,-0.018042)(221,-0.009021)(231,-0.009021)(241,-0.018042)(251,-0.000000)(261,-0.036084)(271,-0.018042)(281,-0.027063)(291,-0.027063)(301,-0.018042)(311,-0.009021)(321,-0.036084)(331,-0.018042)(341,-0.018042)(351,-0.000000)(361,-0.018042)(371,-0.009021)(381,-0.000000)(391,-0.018042)(401,-0.027063)(411,-0.009021)(421,-0.018042)(431,0.000000)(441,-0.009021)(451,-0.018042)(461,-0.009021)(471,-0.000000)(481,-0.018042)(491,0.000000)(501,0.000000)(511,0.018042)(521,-0.036084)(531,-0.027063)(541,-0.018042)(551,-0.018042)(561,-0.027063)(571,-0.009021)(581,-0.027063)(591,-0.009021)(601,-0.018042)(611,-0.018042)(621,-0.009021)(631,-0.000000)(641,-0.036084)(651,-0.018042)(661,-0.018042)(671,-0.000000)(681,-0.018042)(691,-0.009021)(701,-0.000000)(711,-0.009021)(721,-0.018042)(731,-0.000000)(741,-0.009021)(751,0.009021)(761,0.000000)(771,-0.027063)(781,-0.018042)(791,-0.009021)(801,-0.027063)(811,-0.009021)(821,-0.009021)(831,0.009021)(841,-0.018042)(851,-0.009021)(861,-0.000000)(871,0.000000)(881,-0.009021)(891,0.009021)(901,-0.018042)(911,0.000000)(921,-0.009021)(931,-0.009021)(941,0.000000)(951,0.009021)(961,-0.018042)(971,0.000000)(981,0.000000)(991,0.018042)(1001,0.000000)(1011,0.009021)(1021,0.018042)(1031,-0.027063)(1041,-0.036084)(1051,-0.018042)(1061,-0.027063)(1071,-0.009021)(1081,-0.018042)(1091,-0.027063)(1101,-0.018042)(1111,-0.009021)(1121,-0.027063)(1131,-0.009021)(1141,-0.009021)(1151,0.009021)(1161,-0.027063)(1171,-0.018042)(1181,-0.009021)(1191,-0.009021)(1201,-0.018042)(1211,0.000000)(1221,-0.018042)(1231,0.000000)(1241,-0.009021)(1251,-0.009021)(1261,0.000000)(1271,0.009021)(1281,-0.036084)(1291,-0.018042)(1301,-0.018042)(1311,-0.000000)(1321,-0.018042)(1331,-0.009021)(1341,0.000000)(1351,-0.009021)(1361,-0.018042)(1371,0.000000)(1381,-0.009021)(1391,0.009021)(1401,0.000000)(1411,-0.018042)(1421,-0.009021)(1431,0.000000)(1441,-0.018042)(1451,0.000000)(1461,0.000000)(1471,0.018042)(1481,-0.009021)(1491,0.000000)(1501,0.009021)(1511,0.009021)(1521,0.000000)(1531,0.018042)(1541,-0.027063)(1551,-0.009021)(1561,-0.018042)(1571,-0.018042)(1581,-0.009021)(1591,-0.000000)(1601,-0.027063)(1611,-0.009021)(1621,-0.009021)(1631,0.009021)(1641,-0.009021)(1651,-0.000000)(1661,0.009021)(1671,-0.009021)(1681,-0.018042)(1691,0.000000)(1701,-0.009021)(1711,0.009021)(1721,0.000000)(1731,-0.009021)(1741,0.000000)(1751,0.009021)(1761,-0.009021)(1771,0.009021)(1781,0.009021)(1791,0.027063)(1801,-0.018042)(1811,-0.009021)(1821,0.000000)(1831,0.000000)(1841,-0.009021)(1851,0.009021)(1861,-0.009021)(1871,0.009021)(1881,0.000000)(1891,0.000000)(1901,0.009021)(1911,0.018042)(1921,-0.018042)(1931,0.000000)(1941,0.000000)(1951,0.018042)(1961,0.000000)(1971,0.009021)(1981,0.018042)(1991,0.009021)(2001,0.000000)(2011,0.018042)(2021,0.009021)(2031,0.027063)(2041,0.018042)(2051,-0.036084)(2061,-0.027063)(2071,-0.018042)(2081,-0.036084)(2091,-0.018042)(2101,-0.018042)(2111,-0.000000)(2121,-0.027063)(2131,-0.018042)(2141,-0.009021)(2151,-0.009021)(2161,-0.018042)(2171,-0.000000)(2181,-0.027063)(2191,-0.009021)(2201,-0.018042)(2211,-0.018042)(2221,-0.009021)(2231,-0.000000)(2241,-0.027063)(2251,-0.009021)(2261,-0.009021)(2271,0.009021)(2281,-0.009021)(2291,-0.000000)(2301,0.009021)(2311,-0.018042)(2321,-0.027063)(2331,-0.009021)(2341,-0.018042)(2351,-0.000000)(2361,-0.009021)(2371,-0.018042)(2381,-0.009021)(2391,-0.000000)(2401,-0.018042)(2411,0.000000)(2421,-0.000000)(2431,0.018042)(2441,-0.018042)(2451,-0.009021)(2461,-0.000000)(2471,0.000000)(2481,-0.009021)(2491,0.009021)(2501,-0.009021)(2511,0.009021)(2521,0.000000)(2531,0.000000)(2541,0.009021)(2551,0.018042)(2561,-0.036084)(2571,-0.018042)(2581,-0.018042)(2591,-0.000000)(2601,-0.018042)(2611,-0.009021)(2621,-0.000000)(2631,-0.009021)(2641,-0.018042)(2651,-0.000000)(2661,-0.009021)(2671,0.009021)(2681,-0.000000)(2691,-0.018042)(2701,-0.009021)(2711,-0.000000)(2721,-0.018042)(2731,0.000000)(2741,-0.000000)(2751,0.018042)(2761,-0.009021)(2771,-0.000000)(2781,0.009021)(2791,0.009021)(2801,-0.000000)(2811,0.018042)(2821,-0.018042)(2831,0.000000)(2841,-0.009021)(2851,-0.009021)(2861,0.000000)(2871,0.009021)(2881,-0.018042)(2891,0.000000)(2901,-0.000000)(2911,0.018042)(2921,0.000000)(2931,0.009021)(2941,0.018042)(2951,0.000000)(2961,-0.009021)(2971,0.009021)(2981,0.000000)(2991,0.018042)(3001,0.009021)(3011,0.000000)(3021,0.009021)(3031,0.018042)(3041,0.000000)(3051,0.018042)(3061,0.018042)(3071,0.036084)(3081,-0.027063)(3091,-0.018042)(3101,-0.009021)(3111,-0.009021)(3121,-0.018042)(3131,-0.000000)(3141,-0.018042)(3151,-0.000000)(3161,-0.009021)(3171,-0.009021)(3181,0.000000)(3191,0.009021)(3201,-0.027063)(3211,-0.009021)(3221,-0.009021)(3231,0.009021)(3241,-0.009021)(3251,-0.000000)(3261,0.009021)(3271,-0.000000)(3281,-0.009021)(3291,0.009021)(3301,0.000000)(3311,0.018042)(3321,0.009021)(3331,-0.018042)(3341,-0.009021)(3351,-0.000000)(3361,-0.018042)(3371,0.000000)(3381,-0.000000)(3391,0.018042)(3401,-0.009021)(3411,-0.000000)(3421,0.009021)(3431,0.009021)(3441,0.000000)(3451,0.018042)(3461,-0.009021)(3471,0.009021)(3481,0.000000)(3491,0.000000)(3501,0.009021)(3511,0.018042)(3521,-0.009021)(3531,0.009021)(3541,0.009021)(3551,0.027063)(3561,0.009021)(3571,0.018042)(3581,0.027063)(3591,-0.009021)(3601,-0.018042)(3611,-0.000000)(3621,-0.009021)(3631,0.009021)(3641,-0.000000)(3651,-0.009021)(3661,0.000000)(3671,0.009021)(3681,-0.009021)(3691,0.009021)(3701,0.009021)(3711,0.027063)(3721,-0.009021)(3731,-0.000000)(3741,0.009021)(3751,0.009021)(3761,0.000000)(3771,0.018042)(3781,0.000000)(3791,0.018042)(3801,0.009021)(3811,0.009021)(3821,0.018042)(3831,0.027063)(3841,-0.018042)(3851,0.000000)(3861,-0.000000)(3871,0.018042)(3881,0.000000)(3891,0.009021)(3901,0.018042)(3911,0.009021)(3921,0.000000)(3931,0.018042)(3941,0.009021)(3951,0.027063)(3961,0.018042)(3971,0.000000)(3981,0.009021)(3991,0.018042)(4001,0.000000)(4011,0.018042)(4021,0.018042)(4031,0.036084)(4041,0.009021)(4051,0.018042)(4061,0.027063)(4071,0.027063)(4081,0.018042)(4091,0.036084)
        };
        \end{axis}
  \end{tikzpicture}
\end{minipage}
 \caption{ \textbf{(Left)} Histogram of singular values of $\bar \A/\sqrt{\bar p (1- \bar p)}$ ({\BLUE \textbf{blue}}) versus the limiting quarter-circle law spectrum and spikes ({\RED \textbf{red}}). 
 \textbf{(Right)} Left singular vector associated to the largest singular value of $\bar \A$ ({\BLUE \textbf{blue}}), versus the (rescaled, according to \Cref{theo:bar_A}) top left singular vector of $\bS_K^{\bPi = \I_N}$ ({\RED \textbf{red}}).
 A similar observation can be made for right singular vectors, but with larger random fluctuation.
 With $m = 2$, $K= 12$ so that $N = m^K =  4\,096$, $p = 0.7$ and ${\rm vec}(\X) = [-5.5,5.5,-1.5,1.5]^\T$.
 }
\label{fig:spectrum}
\end{figure}

The asymptotic behavior of the singular values and vectors in \Cref{theo:bar_A} are numerically confirmed in \Cref{fig:spectrum} for $K = 12$ and $N=4\,096$.
We observe, in the case of \Cref{fig:spectrum}, that one singular value of $\bar \A$ (due to the small-rank $\bS_K^{\bPi=\I_N}$) isolates from the limiting quarter-circle law, with the associated singular vector a noisy and rescaled version of that of $\bS_K^{\I_N}$.
We also see that the top singular vector of $\bS_K^{\I_N}$ establishes a clear pattern, as a consequence of the linear relation in \eqref{eq:S_linear_propo}.
This property will be exploited later for approximate inference of the graph parameters $\X$.

\begin{Remark}[On small-rank perturbation of random matrices]\normalfont
\label{rem:small_rank}
The spiked model of the form $\bar \A$ in \Cref{prop:approx_centered_adj} has attracted significant research interest in the literature of large-dimensional random matrix theory, see for example~\cite{johnstone2001distribution,baik2006eigenvalues,benaych2012singular,bai2012sample}.
To the best of our knowledge, the only previous efforts that have studied the case of small but increasing rank (with $\rank(\bS) = o(N)$ for $\bS$ the signal matrix) are \cite{peche2006largesta} for deformed complex Gaussian Wigner matrices and \cite{huang2018Mesoscopic} under both additive and multiplicative perturbation models of the type $\Z \Z^\T + \bS \bS^\T$ or $(\I + \bS)^{\frac12} \Z \Z^\T (\I + \bS)^{\frac12}$ for random $\Z$ and small-rank signal $\bS$. 
However, these results do not directly as the model under study here in different from that in~\cite{huang2018Mesoscopic}.
In this vein, we extend the technical results in~\cite{huang2018Mesoscopic} to characterize the adjacency singular spectra of random Kronecker graph models as in \Cref{def:random_Kronecker_graph}.
\end{Remark}

In the following result, we provide asymptotic theoretical guarantee on the shrinkage estimator used in \Cref{alg:shrinkage_estim}, by adapting the proof from~\cite{gavish2017optimal} to the Kronecker graph model.

\begin{Corollary}[Shrinkage estimation of small-rank $\bS_K$]\label{coro:opt_shrinkage}
Under the notation and setting of \Cref{theo:bar_A}, define the following shrinkage estimator, 
\begin{equation}\label{eq:def_hat_S_K}
    \hat \bS_K = \textstyle \sum_{i=1}^N f(\hat \sigma_i) \hat \uu_i \hat \vv_i^\T,
\end{equation}
for $f(t) = \sqrt{t^2 - 4 \bar p (1- \bar p) } \cdot 1_{t > 2 \sqrt{\bar p (1 - \bar p)}}$ and $(\hat\sigma_i, \hat \uu_i, \hat \vv_i)$ the triple of singular values (listed in a decreasing order) and  singular vectors of $\bar \A$.
Suppose all singular values of $\bS_K^{\bPi}$ that are greater than $\sqrt{ \bar p(1- \bar p) }$ are all distinct, one has
\begin{equation*}
     \textstyle \| \bS_K^{\bPi} - \hat \bS_K \|_F^2 - \sum_{i=1}^{\rank(\bS_K)} g \left( \sigma_i (\bS_K^{\bPi}) \right) \to 0,
\end{equation*}
almost surely as $N \to \infty$, with
\begin{equation*}
  g(t) = 
  \begin{cases}
    \bar p (1 - \bar p) \left(2 - \bar p (1 - \bar p) t^{-2} \right), &  t >  \sqrt{\bar p(1 - \bar p)} \\ 
    t^2, & t \leq \sqrt{\bar p(1 - \bar p)}.
  \end{cases}
\end{equation*}
\end{Corollary}
\begin{proof}[Proof of \Cref{coro:opt_shrinkage}]
Here we prove \Cref{coro:opt_shrinkage} following the line of arguments in as in the proof of \cite[Theorem~1]{gavish2017optimal}.
Note that by expanding the Frobenius norm, we get, for $\sigma_i \equiv \sigma_i(\bS_K^{\bPi})$ the ordered singular values of $\bS_K^{\bPi}$ and $f(t)$ defined in \eqref{eq:def_hat_S_K} that, 
\begin{align*}
    &\| \bS_K^{\bPi} - \hat \bS_K \|_F^2 = \textstyle \sum_{i=1}^r \left[ \left( \sigma_i \right)^2 + \left( f(\hat \sigma_i) \right)^2 \right] - 2 \sum_{i,j=1}^r \sigma_i f(\hat \sigma_i) (\uu_i^\T \hat \uu_j) (\vv_i^\T \hat \vv_j) +o(1) \\ 
    &= \textstyle \sum_{i=1}^r \left[ \left( \sigma_i \right)^2 - 2 \sigma_i f(\hat \sigma_i) (\uu_i^\T \hat \uu_i)  (\vv_i^\T \hat \vv_i) + \left( f(\hat \sigma_i) \right)^2 \right] + o(1), \\ 
    &= \textstyle \sum_{i=1}^r \left[ \bar p (1 - \bar p) (2 - \bar p (1 - \bar p) \sigma_i^{-2} ) \cdot 1_{\sigma_i > \sqrt{\bar p(1 - \bar p)}} + \sigma_i^2 \cdot 1_{\sigma_i \leq \sqrt{\bar p(1 - \bar p)}} \right] + o(1)
\end{align*}
where we used in the first equality the fact that there are at most $r \equiv \rank(\bS_K)$ singular values $\hat \sigma_i$ of $\bar \A$ greater than (the right edge of the quarter-circle law) $2 \sqrt{\bar p (1 - \bar p)}$ by \Cref{theo:bar_A}, and the asymptotic singular vector characterization in \Cref{theo:bar_A} in the second and third line.
It can be shown that the nonlinear shrinkage estimator $\hat \bS_K$ introduced in \Cref{coro:opt_shrinkage} yields the \emph{minimum} (asymptotic) Frobenius norm error among all estimators of the form $\hat \bS_K = \textstyle \sum_{i=1}^N f(\hat \sigma_i) \hat \uu_i \hat \vv_i^\T$ with $f\colon \RR_{\geq 0} \to \RR_{\geq 0}$, see for detail in \cite[Theorem~1]{gavish2017optimal}.
\end{proof}

\section{Additional Numerical Results}
\label{sec:additional_num}


\Cref{fig:compare_convex_moment} compares the performance and running time of the proposed \Cref{algo:meta}, the moment-based approach proposed in~\cite{gleich2012MomentBased}, and the KronFit algorithm in~\cite{leskovec2010kronecker}. 

\begin{Remark}[On moment-based method]\normalfont
\label{rem:moment}
The moment-based approach proposed in~\cite{gleich2012MomentBased} has the following limitations in Kronecker graph inference:
\begin{enumerate}
    \item[(i)] it is applicable only when the Kronecker initiator $\bP_1$ has a dimension of $m = 2$; and
    \item[(ii)] it only applies to undirected graphs. 
\end{enumerate}
It employs three strategies to solve for the Kronecker graph initiator: the direct minimization approach, the grid-search approach, and the leading-term-matching approach. 

The direct minimization approach is (believed to be) able to achieve similar performance as the grid-search approach, albeit with significantly reduced computational time, see \cite[Section~5.2]{gleich2012MomentBased}.
Conversely, the leading-term-matching approach can offer a noteworthy computational speed advantage, owing to its distinctive solution methodology. 
Its use cases are limited to Kronecker graphs satisfying some technical conditions, which may not always hold in practical scenarios, as detailed in \cite[Section~4.3]{gleich2012MomentBased}. 
\end{Remark}

As a consequence of the discussions in \Cref{rem:moment}, we adopt the direct minimization procedure when employing the moment-based approach, and test these methods on undirected Kronecker graphs.
We observe from \Cref{fig:compare_convex_moment} that:
\begin{itemize}
  \item[(i)]  for sparse graphs, the proposed \Cref{algo:meta} and the KronFit algorithm outperform the moment-based method; while for dense graphs, the moment-based approach exhibits a slight performance advantage over the two approaches; and
  \item[(ii)] the moment-based approach demonstrates a running time much lower than KronFit and even than standard \Cref{algo:meta} for sparse graphs (in fact even to that of accelerated \Cref{algo:meta} using RNLA techniques, so \Cref{rem:complexity}~and~\Cref{fig:accelerated} for further discussions and illustrations); and
  \item[(iii)] the running time of the moment-based approach, however, grows rapidly as the graph becomes denser, while the running time of the proposed approach stays within a reasonably acceptable range. 
\end{itemize}

 
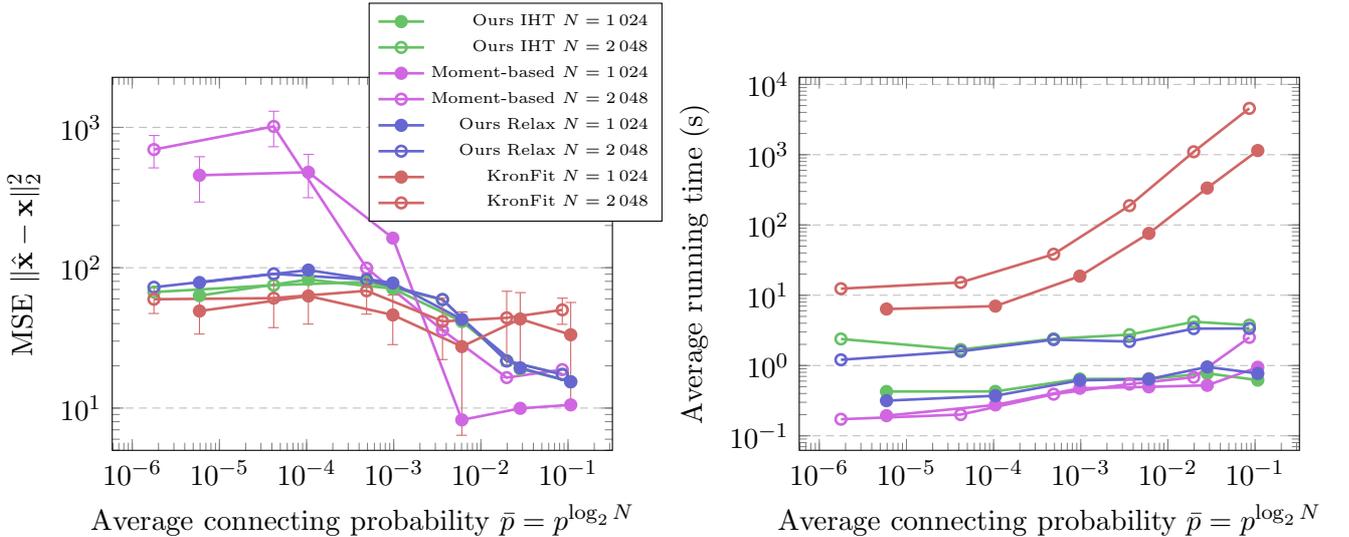
\begin{figure}[thb]
  \centering
  \begin{minipage}[b]{0.48\textwidth}
  \begin{tikzpicture}
\renewcommand{\axisdefaulttryminticks}{4} 
\pgfplotsset{every major grid/.style={densely dashed}}       
\tikzstyle{every axis y label}+=[yshift=-10pt] 
\tikzstyle{every axis x label}+=[yshift=5pt]
\pgfplotsset{every axis legend/.append style={cells={anchor=east},fill=white, at={(1.1,1.2)}, anchor=north east, font=\tiny}}
\begin{axis}[
width=\columnwidth,
height=.8\columnwidth,
ymin = 5,
ymajorgrids=true,
scaled ticks=true,
xlabel = { Average connecting probability $\bar p = p^{\log_2 N}$ },
ylabel = { MSE $\| \hat \x - \x \|_2^2$ },
scaled ticks=true,
ymode=log,
xmode=log
]
\addplot[mark=*,color=GREEN!60!white,line width=1pt] coordinates{
(0.300000^10,   63.2729 )(0.400000^10,  82.3895)(0.500000^10, 70.9075)(0.600000^10,   41.3417  )(0.700000^10, 19.0156)(0.800000^10,15.4009)
};
\addlegendentry{{Ours IHT $N = 1\,024$ }};
\addplot[mark=o,color=GREEN!60!white,line width=1pt] coordinates{
(0.300000^11,   67.0518)(0.400000^11, 75.0707)(0.500000^11, 79.2137)(0.600000^11,  59.2202)(0.700000^11, 21.6673)(0.800000^11,17.425)
};
\addlegendentry{{Ours IHT $N = 2\,048$ }};
\addplot[
mark=*,color=PURPLE!60!white,line width=1pt,
error bars/.cd,
y dir=both, y explicit,
] 
coordinates{
(0.300000^10, 455.8554 ) +- (0,162.0216)
(0.400000^10, 478.7571) +- (0,163.3968)
(0.500000^10,162.5238) +- (0,0.001616)
(0.600000^10, 8.2383) +- (0,0.0010407)
(0.700000^10,  9.957  ) +- (0,0.0021972)
(0.800000^10,10.5416) +- (0,0.0013998)
};
\addlegendentry{{Moment-based $N = 1\,024$ }};
\addplot[
mark=o,color=PURPLE!60!white,line width=1pt,
error bars/.cd,
y dir=both, y explicit,
] 
coordinates{
(0.300000^11,694.7213) +- (0,180.3916)
(0.400000^11, 1015.2309) +- (0,287.5895)
(0.500000^11,99.3327) +- (0,0.0088801)
(0.600000^11, 35.8615) +- (0,	0.073546)
(0.700000^11, 16.4876 ) +- (0,	0.0019986)
(0.800000^11, 18.8002  ) +- (0,0.0026839)
};
\addlegendentry{{Moment-based $N = 2\,048$ }};
\addplot[mark=*,color=BLUE!60!white,line width=1pt] coordinates{
(0.300000^10,78.4424)(0.400000^10,96.2016)(0.500000^10,77.7389)(0.600000^10,42.8566)(0.700000^10,19.3596)(0.800000^10,15.3863)
};
\addlegendentry{{ Ours Relax $N = 1\,024$ }};
\addplot[mark=o,color=BLUE!60!white,line width=1pt] coordinates{
(0.300000^11,72.5987)(0.400000^11,90.4576)(0.500000^11,82.3157)(0.600000^11,59.3903)(0.700000^11,21.6684)(0.800000^11,17.3659)
};
\addlegendentry{{ Ours Relax $N = 2\,048$ }};
\addplot[
mark=*,color=RED!60!white,line width=1pt,
error bars/.cd,
y dir=both, y explicit,
] 
coordinates{
(0.300000^10, 49.1478) +- (0,15.4691)
(0.400000^10, 62.8903) +- (0,23.2333)
(0.500000^10,  46.0704) +- (0,17.7445)
(0.600000^10,  27.4228) +- (0,21.0177)
(0.700000^10,   43.1384  ) +- (0,23.3751)
(0.800000^10, 33.2888) +- (0,23.2627)
};
\addlegendentry{{KronFit $N = 1\,024$ }};
\addplot[
mark=o,color=RED!60!white,line width=1pt,
error bars/.cd,
y dir=both, y explicit,
] 
coordinates{
(0.300000^11,59.6146) +- (0,12.4059)
(0.400000^11, 60.7025) +- (0,23.2638)
(0.500000^11,  68.5044) +- (0,	21.8136)
(0.600000^11,  41.4569) +- (0,	19.2902)
(0.700000^11, 43.9688 ) +- (0,	23.9756)
(0.800000^11,  50.0855) +- (0,10.5742)
};
\addlegendentry{{KronFit $N = 2\,048$ }};
\end{axis}
\end{tikzpicture}
  \end{minipage}
  \hfill{}%
  \begin{minipage}[b]{0.48\textwidth}
\begin{tikzpicture}
\renewcommand{\axisdefaulttryminticks}{4} 
\pgfplotsset{every major grid/.style={densely dashed}}       
\tikzstyle{every axis y label}+=[yshift=-10pt] 
\tikzstyle{every axis x label}+=[yshift=5pt]
\pgfplotsset{every axis legend/.append style={cells={anchor=east},fill=white, at={(0.02,0.00)}, anchor=north west, font=\tiny }}
\begin{axis}[
width=\columnwidth,
height=.8\columnwidth,
xlabel = { Average connecting probability $\bar p = p^{\log_2 N}$ },
ymode=log,
xmode=log,
ylabel = { Average running time (s) },
ymajorgrids=true,
scaled ticks=true,
]
\addplot[mark=*,color=GREEN!60!white,line width=1pt] coordinates{
(0.300000^10,0.4268)(0.400000^10, 0.4282)(0.500000^10,0.6441)(0.600000^10, 0.6471)(0.700000^10,0.7736)(0.800000^10,0.6200)
};

\addplot[mark=o,color=GREEN!60!white,line width=1pt] coordinates{
(0.300000^11, 2.3875)(0.400000^11,  1.6876)(0.500000^11,2.4012)(0.600000^11, 2.7513)(0.700000^11,4.1941)(0.800000^11,3.7581 )
};

\addplot[mark=*,color=PURPLE!60!white,line width=1pt] coordinates{
(0.300000^10,0.193400)(0.400000^10,0.275000)(0.500000^10,0.475000)(0.600000^10,0.496900)(0.700000^10,0.522600)(0.800000^10,0.945100)
};
\addplot[mark=o,color=PURPLE!60!white,line width=1pt] coordinates{
(0.300000^11,0.172100)(0.400000^11,0.200400)(0.500000^11,0.391100)(0.600000^11,0.544100)(0.700000^11,0.680100)(0.800000^11,2.5110)
};
\addplot[mark=*,color=BLUE!60!white,line width=1pt] coordinates{
(0.300000^10, 0.316300)(0.400000^10,0.371200)(0.500000^10,0.61600)(0.600000^10, 0.639600 )(0.700000^10,0.955100)(0.800000^10,0.770700)
};
\addplot[mark=o,color=BLUE!60!white,line width=1pt] coordinates{
(0.300000^11,1.206800)(0.400000^11, 1.592000 )(0.500000^11, 2.334600)(0.600000^11,2.195300)(0.700000^11,3.348400)(0.800000^11,3.369100)
};
\addplot[mark=*,color=RED!60!white,line width=1pt] coordinates{
(0.300000^10, 6.3676)(0.400000^10,6.9988)(0.500000^10,18.7124)(0.600000^10, 75.8326 )(0.700000^10,335.5459)(0.800000^10,1145.1141)
};
\addplot[mark=o,color=RED!60!white,line width=1pt] coordinates{
(0.300000^11,12.3988)(0.400000^11, 15.1995 )(0.500000^11, 38.3759)(0.600000^11,188.4528)(0.700000^11,1100.3826)(0.800000^11,4556.6589)
};
\end{axis}
\end{tikzpicture}
\end{minipage}
\caption{{Estimation MSEs (\textbf{left}) and running time (\textbf{right}) of the moment-based method, the KronFit algorithm and the proposed approach on random undirected Kronecker graphs
with $p$ ranging from $0.3$ to $0.8$, $\x = [4.75, 1.75, 1.75, -8.25]$, and $20\%$ vertices randomly shuffled, for $N = 1\,024$ and $2\,048$. Result obtained over $10$ independent runs.}}
\label{fig:compare_convex_moment}
\end{figure}

\section{Dataset Statistics}
\label{sec:datasets}

We present \Cref{table::dataset_statistics} below the statistics of the graph classification datasets used in \Cref{subsec:application_realistic}.
PROTEINS \cite{borgwardt2005protein}, NCI1 \cite{wale2008comparison}, and ENZYMES are chemical graphs, whereas IMDB-B, REDDIT-B, COLLAB, IMDB-M, and REDDIT-5K are social graphs.

\begin{table*}[h!] 
  \caption{Statistics of different graph datasets from~\cite{KKMMN2016}} 
  \label{table::dataset_statistics}
  \centering
  \begin{tabular}{lcccccccc}
    \toprule
    Dataset & Graphs & Classes & Average Number of Nodes & Average Number of Edges  \\
    \midrule
    PROTEINS & $1\,113$ & $2$ & $39.06$ & $72.82$ \\
    NCI1  & $4\,110$ & $2$ & $29.87$ & $32.30$ \\
    REDDIT-B & $2\,000$ & $2$ & $429.63$ & $497.75$ \\
    IMDB-B  & $1\,000$ & $2$ & $19.77$ & $96.53$ \\
    ENZYMES & $600$ & $6$ & $32.63$ & $62.14$ \\
    COLLAB & $5\,000$ & $3$ & $74.49$ & $2457.78$ \\
    IMDB-M & $1\,500$ & $3$ & $13.00$ & $65.94$ \\
    REDDIT-5K  & $4\,999$ & $5$ & $508.52$ & $594.87$ \\
    \bottomrule
  \end{tabular}
\end{table*}
\end{document}